\documentclass[10pt,article,onecolumn]{IEEEtran}
\usepackage{amssymb,amsmath,amsthm}
\usepackage{graphicx,epstopdf}
\newtheorem{theorem}{Theorem}[section]
\newtheorem{lemma}[theorem]{Lemma}
\newtheorem{definition}[theorem]{Definition}

\newtheorem{remark}{Remark}
\newtheorem{example}{Example}
\newtheorem{assumption}{Assumption}

\newcommand{\projectionMD}[1]{\Pi_{\MD}\brackets{#1}}
\newcommand{\MDK}{\MD^K}

\newcommand{\limn}{\lim_{n \to \infty}}

\newcommand{\problemA}{ {\bf A} }
\newcommand{\problemB}{ {\bf B} }

\newcommand{\problemCK}{ {\bf C(K)} }
\newcommand{\problemDKn}{ {\bf D(K,n)} }

\newcommand{\bitm}{\begin{itemize}}
\newcommand{\eitm}{\end{itemize}}

\newcommand{\benum}{\begin{enumerate}}
\newcommand{\eenum}{\end{enumerate}}
\newcommand{\kldist}[2]{D\! \left( #1\|#2 \right)}
\newcommand{\reals}{\mathbb{R}}
\newcommand{\parenth}[1] {\left(#1\right)}

\newcommand{\braces}[1] {\left\{#1\right\}}
\newcommand{\brackets}[1] {\left[#1\right]}
\newcommand{\abs}[1] {\left|#1\right|}
\newcommand{\norm}[1] {\left\|{#1}\right\|}

\newcommand{\beqa}{\begin{eqnarray}}
\newcommand{\eeqa}{\end{eqnarray}}
\newcommand{\beqas}{\begin{eqnarray*}}
\newcommand{\eeqas}{\end{eqnarray*}}
\newcommand{\prob}[1] {\mathbb{P}\parenth{#1}}

\newcommand{\E}{\mathbb{E}}

\def\argmin{\mathop{\arg\,\!\min}\limits}%
\def\argmax{\mathop{\arg\,\!\max}\limits}%
\newcommand{\innerprod}[1] {\langle#1\rangle}
\newcommand{\indicatorvbl}[1] {1_{\braces{#1}}}
\newcommand{\alphabet}[1] { {\mathsf #1}}
\newcommand{\probSimplex}[1]{ \alphabet{P}\parenth{\alphabet{#1}}}

\newcommand{\cX}{\alphabet{X}}
\newcommand{\cY}{\alphabet{Y}}

\newcommand{\tS}{\tilde{S}}

\newcommand{\diffeomorphisms}{\mathcal{D}}

\newcommand{\pushforward}[3]{ {#1 {\#} #2 = #3} }

\newcommand{\obj}{V}

\newcommand{\jac}{J}
\newcommand{\absdet}[2]{\abs{\det{\parenth{\jac_{#1}(#2)}}}}
\newcommand{\absdetSx}{\absdet{S}{x}}

\newcommand{\logdetJSx}{\log \det \parenth{J_{S}(x)}}

\newcommand{\basisfn}[1]{\phi^{(#1)}}

\renewcommand{\P}{\mathbb{P}}

\newcommand{\WKn}{{W^*_{K,n}}}

\newcommand{\tP}{\tilde{P}}
\newcommand{\tp}{\tilde{p}}

\newcommand{\MD}{\mathcal{D}_+}

\renewcommand{\P}{\mathbb{P}}
\newcommand{\SK}{S^*_K}
\newcommand{\Sinf}{S_{\infty}}

\newcommand{\SKn}{S^*_{K,n}}

\title{Tractable Fully Bayesian Inference  via Convex Optimization and Optimal Transport Theory}
\author{Sanggyun Kim, Diego Mesa, Rui Ma, and Todd P. Coleman
\thanks{S. Kim, D. Mesa and T. P. Coleman are with the Department
of Bioengineering, University of California San Diego, La Jolla,
CA, 92093 USA e-mail: s2kim@ucsd.edu, damesa@ucsd.edu,
tpcoleman@ucsd.edu.}
\thanks{R. Ma is with Dexcom Inc., San Diego, CA, 92121 USA e-mail: rma@dexcom.com.}}

\begin{document}

\maketitle

\pagestyle{plain}

\begin{abstract}

We consider the problem of transforming samples from one
continuous source distribution into samples from another target
distribution. We demonstrate with optimal transport theory that
when the source distribution can be easily sampled from and the
target distribution is log-concave, this can be tractably solved
with convex optimization. We show that a special case of this,
when the source is the prior and the target is the posterior, is
Bayesian inference. Here, we can tractably calculate the
normalization constant and draw posterior \textit{i.i.d}. samples.
Remarkably, our Bayesian tractability criterion is simply log
concavity of the prior and  likelihood: the same criterion for
tractable calculation of the maximum a posteriori point estimate.
With simulated data, we demonstrate how we can attain the Bayes
risk in simulations. With physiologic data, we demonstrate
improvements over point estimation in intensive care unit outcome
prediction and electroencephalography-based sleep staging.
\end{abstract}

\newcommand{\itt}[1]{\textit{#1}}
\newcommand{\bt}[1]{\textbf{#1}}

\section{Introduction}
\label{sec:intro}

Reasoning about data in the presence of noise or incomplete
knowledge is fundamental to fields as diverse as science,
engineering, medicine, and finance. One natural and principled way
to reason with uncertainty is Bayesian inference, where a prior
distribution $P_X$ about a random variable $X$ is combined with a
likelihood model $P_{Y|X=x}$ and a measurement $y$ to obtain an
\itt{a posteriori} distribution, specified by Bayes' rule: \beqa
p(x|y) = \frac{p(x)p(y|x)}{\beta_y}, \qquad \quad \beta_y
\triangleq \int_u p(y|u)p(u) du.  \label{eqn:BayesRule} \eeqa The
posterior completely represents our knowledge about $X$.

In many machine learning architectures, a single point estimate on
$X$ is obtained, such as the maximum a posteriori (MAP) point
estimate given by \beqa
\hat{x}_{MAP} &=& \argmax_x \;\;p(x|y)  \nonumber \\
&=& \argmax_x \;\;\log p(x) + \log p(y|x) \label{eqn:MAPestimate}.
\eeqa Note that the MAP estimate does not require the calculation
of $\beta_y$; moreover, \eqref{eqn:MAPestimate} is tractable and
can be solved with general purpose convex optimization solvers
when the prior and likelihood are log-concave. Such distributions
are very natural, used across many different domains
\cite{bagnoli2005log}, and include most commonly used models
(exponential family priors, generalized linear model likelihoods,
etc).

However, as shown in Fig.~\ref{fig:conceptual-maxapos}, the MAP
estimate can be ambiguous: the same MAP estimate can pertain to
one posterior with less variance (less information gain) than
another.
Decision-making using only a point estimate, without taking into
account variability, can have adverse effects in many different
situations, including system design, fault tolerance, and
classification \cite{walker2003defining}.

In order to perform optimal decision-making, one must calculate
conditional expectations of interest for minimizing expected loss
and attaining the Bayes risk \cite{degroot1970optimal}: \beqa
 \E[l(a,X)|Y=y] \underbrace{\simeq}_{Z_i \;\textit{i.i.d.}~\sim P_{X|Y=y}} \frac{1}{n} \sum_{i=1}^n l(a,Z_i). \label{eqn:conditionalExpectation}
 \eeqa
Also, $\beta_y$ is used to calculate information measures
involving log likelihood ratios, such as the posterior information
gain from measuring $y$ \cite{raginsky2009mutual}
\begin{align}
 G(P_{X|Y=y},P_X,) &= \kldist{P_{X|Y=y}}{P_X} \nonumber \\
  &=  -\log \beta_y  + \E \brackets{\;\log p(y|X)\;\big|Y=y}\label{eqn:posteriorInformationGain}
 \end{align}
 for sequential experiment design \cite{degroot1962uncertainty}, or the mutual information $\int_{y \in \cY} G(P_X,P_{X|Y=y}) P_Y(y)dy$ \cite{cover2006elements}.  Conditional information measures such as conditioanl mutual information \cite{koller2009probabilistic} and directed information \cite{quinn2015directedInformationGraphs} also build upon posterior information gains and are relevant for building graphical models to succinctly represent conditional indepencences in data.   Outside of specialized, problem-specific Monte Carlo methods, uncertainty quantification
can be tractably and accurately performed when $\dim{\cX}$ is
small
 (e.g. credibility intervals) or $\dim{\cY}$ is large
 (e.g. Laplace's method using Gaussian approximations \cite{kass1988asymptotics,geisser1990validity}).
 Here, we will develop a tractable, general-purpose framework
 for uncertainty quantification in Bayesian inference within the context of log-concavity.

\begin{figure}[t]
\centering
\includegraphics[width=8cm]{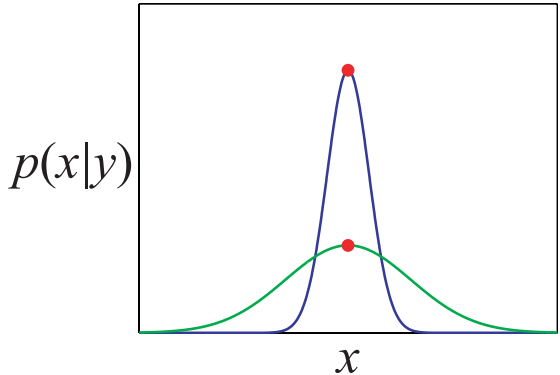}
\caption{Two posterior distributions with the same MAP estimate
but different variances. Decision-making using point estimates but
without uncertainty quantification can significantly reduce
performance as compared to using the full posterior.}
\label{fig:conceptual-maxapos}
\end{figure}

\subsection{Related work}
\label{subsec:relatedwork} The full exploitation of Bayesian
inference has been developed in a wide range of problems.

Hierarchical Bayesian modeling incorporates priors on the
hyper-parameters of the prior distribution to perform high-order
statistical modeling and inference \cite{good1980some,
goldstein2011multilevel, tzikas2008variational}. This  improves
the prediction of a Bayesian model that often depends on the prior
distributions. When $X$ is a random process obeying a state-space
model (e.g. Markov chain), particle filtering and sequential
importance sampling incorporate the  dynamics of $X$
\cite{djuric2003particle, arulampalam2002tutorial} to sequentially
update posteriors. These sequential Monte Carlo methods
recursively compute the relevant probability distributions using
ensembles of ``particles'' and their weights. Nonparametric
Bayesian methods, attracting increasing interest,
\cite{walker1999bayesian, neal2000markov, muller2004nonparametric,
teh2006hierarchical}, make fewer assumptions on the distributions
of interest and allow their parametric complexity to increase with
the amount of data acquired.

For the remainder of this manuscript, we will consider the
canonical Bayesian inference problem in \eqref{eqn:BayesRule}
where $X \subset \reals^d$, $P_X(x)$ has a density $p(x)$ with
respect to the Lebesgue measure, and the likelihood model is
specified in density form as $p(y|x)$.

For a given likelihood $p(y|x)$, in order to obviate the
integration in calculating $\beta_y$ in \eqref{eqn:BayesRule},
conjugate priors allow for a closed-form expression for the
posterior. However,  these are very limiting cases and the
conjugate prior is often selected only for computational purposes,
even if it does not actually represent prior knowledge about $X$.

Markov chain Monte Carlo (MCMC) methods \cite{robert2004monte,
andrieu2003introduction, hastings1970monte, geman1984stochastic,
Liu2008} have enabled widespread attempts at fully Bayesian
inference by representing a probability distribution as a set of
samples and iterating them through a Markov chain whose invariant
distribution is the posterior. Despite the wide adoption, MCMC has
a few drawbacks: (a)  the convergence rates and mixing times of
the Markov chains are generally unknown, thus leading to practical
shortcomings like ``burn in'' periods of discarded samples; (b)
the samples generated are from a Markov chain and thus necessarily
correlated -- lowering  effective sample sizes and propagating
errors throughout  estimates as in
\eqref{eqn:conditionalExpectation}; (c) many different highly
specialized, problem-specific variants of MCMC exist, all of which
have their own benefits and challenges \cite{robert2004monte}.

Efficient approximation methods, such as variational Bayes
\cite{jordan1998introduction, jaakkola2001tutorial,
bishop2006pattern} and expectation propagation (EP)
\cite{minka2001expectation, minka2001family, seeger2008bayesian}
have been developed. These offer a complementary alternative to
sampling methods and have allowed Bayesian techniques to be used
in large-scale applications. Variational methods yield
deterministic approximations to the posterior distribution. A
particular form that has been used with great success is the
factorized one \cite{jordan1999introduction}.  However, these
methods are based upon approximations and thus there are no
guarantee that the iterations of the approximation methods will
converge, or provide exact results in the limit.

\subsection{Optimal transport framework}
\label{subsec:OTF} Recently, El Moselhy et al. proposed a method
to construct a \itt{map} that \itt{pushed forward} the prior
measure to the posterior measure, casting Bayesian inference as an
optimal transport problem \cite{el2012bayesian}. Namely, the
constructed map
 \itt{transforms} a random variable distributed
according to the prior into another random variable distributed
according to the posterior. This approach is \itt{conceptually}
different from previous methods, including sampling and
approximation methods.

Optimal transport theory has a long and deep history dating back
to Monge \cite{monge1781}, Kontarovich \cite{kantorovich1942mass},
and most recently Villani \cite{villani2003topics}. This study of
measure preserving maps has its roots and applications in vast
areas, spanning resource allocation, dynamical systems, optimal
control, and fluid mechanics.

Monge initially stated the ``earth movers'' problem as finding the
cheapest way to move a pile of sand in a specific location and
shape, to a \itt{different} location and shape while maintaining
the same volume. While this captures the mass-preserving and
optimality criterion of the map, it is worth explaining the
\itt{effect} of the map a little more by way of an example. We can
relate the optimal map construction to finding the optimal
placement of ``pegs'' in a Galton board, to transform one
distribution of balls into another as shown in
Fig.~\ref{fig:plinko}. Conceptually, one could imagine searching
for a unique beam placement resulting in a \itt{different} output
distribution. For example, Fig.~\ref{fig:plinko}~(a) shows the
trivial schematic diagram of a map that pushes forward a uniform
distribution $P$ to another uniform $Q$, while
Fig.~\ref{fig:plinko}~(b) pushes forward a uniform $P$ to a
Gaussian distribution $Q$.

\begin{figure}[t]
\centering
\includegraphics[width=12cm]{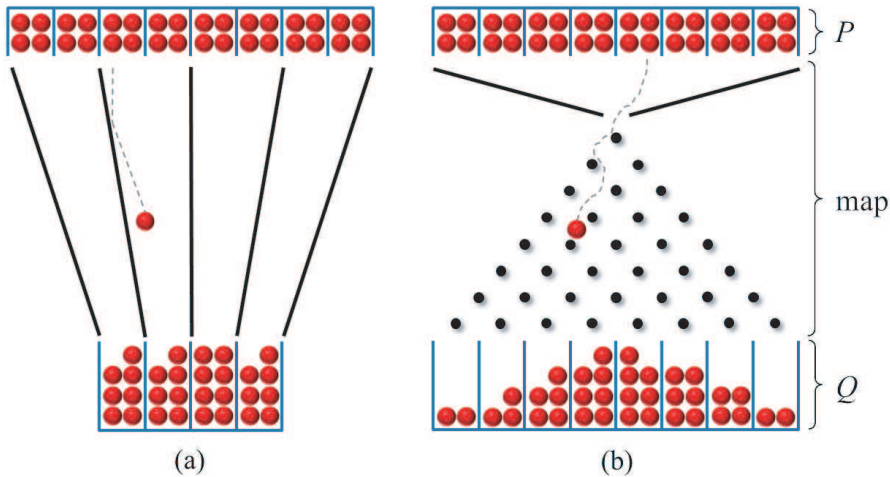}
\caption{Conceptual explanation of transforming samples from a
distribution $P$ to distribution $Q$ through the use of a Galton
board. Two different placements provide some intuition for a
map-based construction: (a) a uniform to another uniform, and (b)
a uniform to a Gaussian.} \label{fig:plinko}
\end{figure}

\subsection{Our contribution}
\label{subsec:contribution} The use of an optimal transport map
for Bayesian inference was proposed in \cite{el2012bayesian} by
minimizing the variance of an operator, but it was a non-convex
problem -- even for the case of log-concavity of prior and
likelihood. In \cite{kim2013efficient}, we consider the optimal
transport map viewpoint established in  \cite{el2012bayesian}, but
we replace variance minimization with an equivalent approach based
upon KL divergence minimization.  Remarkably, for the case of
log-concave priors and likelihoods, we showed this KL divergence
minimization is a convex optimization problem and thus tractable.

The rest of the paper is outlined as follows. In
section~\ref{sec:gen-push-thrm}, we first consider a more general
problem: transforming samples from one continuous source
distribution  into samples from another target distribution. We
demonstrate with optimal transport theory that when the source
distribution can be easily sampled from and the target
distribution is log-concave, a KL divergence minimization
procedure yields samples from target distribution with convex
optimization.  We develop an empirical and truncation approach to
computationally approximate the convex problem.  We exploit the
sub-exponential tail property of log-concave distributions to
prove consistency of the proposed scheme in the remainder of this
section.  In section~\ref{sec:BayesInference}, we  demonstrate
that fully Bayesian inference (e.g. where the source is the prior
and target is the posterior) is a special case.  Remarkably, we
show that the tractability criterion becomes log-concavity of the
prior and likelihood in $x$: the  same criterion for obtaining
tractable MAP estimates in \eqref{eqn:MAPestimate}.  This implies
that  general purpose frameworks for Bayesian point estimation
with convex optimization can be improved to fully Bayesian
inference, still with convex optimization.
Section~\ref{sec:results} demonstrate how we can attain the  Bayes
risk in simulations.  With physiologic data, we demonstrate
improvements over point estimation in intensive care unit (ICU)
outcome prediction and in sleep staging based upon
electroencephalography (EEG) recordings. We conclude with a
discussion in Section~\ref{sec:discussion}.

\section{General push-forward theorem}
\label{sec:gen-push-thrm} Before going into the details, we
present a general theorem of an optimal map construction to
transform a distribution $P$ to another distribution $Q$. We also
informally go through an illustrating example. This general
push-forward theorem will be used in the Bayesian inference
framework in the next section.

\subsection{Problem setup}
\label{subsec:genprobsetup}

We provide a problem setting with notations and definitions
relevant to the development of the push-forward theorem where the
latent variable is in a continuum. For a set $\cX \subset
\reals^d$ where  $d$ is a positive integer, define the space of
all probability measures on $\cX$ as $\probSimplex{X}$. Then the
\itt{push-forward} theorem is defined as follows.

\begin{definition}[Push-forward]
Given $P \in \probSimplex{X}$ and $Q \in \probSimplex{X}$, we say
that the map $S: \cX \to \cX$ \textbf{pushes forward} $P$ to $Q$
(denoted as $\pushforward{S}{P}{Q}$) if a random variable $U$ with
distribution $P$ results in $Z \triangleq S(X)$ having
distribution $Q$.
\end{definition}

We say that $S: \cX \to \cX$ is a \itt{diffeomorphism} on $\cX$ if
$S$ is invertible and both $S$ and $S^{-1}$ are differentiable.
Denote the set of all diffeomorphisms on $\cX$ as
$\diffeomorphisms$. With this, we have the following lemma from
standard probability:

\begin{lemma}
Consider a diffeomorphism $S \in \diffeomorphisms$ and $P$, $Q\in
\probSimplex{X}$ that both have the densities $p$, $q$ with
respect to the Lebesgue measure. Then $S_\# P = Q$ if and only if
\begin{eqnarray}
p(x) = q(S(x)) \absdetSx, \quad \forall x \in \cX
\label{eqn:defn:JacobianEqn}
\end{eqnarray}
where $J_{S}(x)$ is the Jacobian matrix of the map $S$ at $x$.
\end{lemma}

Throughout this section it is worth discussing a simple example:
\begin{example}\label{example:twoOptimalMaps}
Let $X \sim P$ where $P$ is a uniform on $[0,2]$, and consider
building the transformation $Z = S(X)$ so that $Z \sim Q$ where
$Q$ is a uniform on $[0,1]$. The desired transformation is
represented by the solid lines in
Fig.~\ref{fig:image-distributions-maps2}. Note that clearly the
maps $S^*(x)=\frac{1}{2}x$ and $S^*(x) = \frac{1}{2}(2-x)$
transform $P$ to $Q$ and thus satisfy the Jacobian equation
\eqref{eqn:defn:JacobianEqn}. One is increasing, the other is
decreasing in $x$. Note that there are \textbf{multiple maps} that
push  $P$ to $Q$.  An arbitrary map $S$, e.g., $S(x) =
\frac{1}{3}x$, does {\bf not necessarily} push  $P$ to $Q$, but
rather pushes \itt{some} $\tP \neq P$ to $Q$, namely,
$\tP=\text{unif}[0,3]$. This transformation is represented by the
dotted line in Fig.~\ref{fig:image-distributions-maps2}.
\end{example}

\begin{figure}[t]
\centerline{\includegraphics[width=6.5cm]{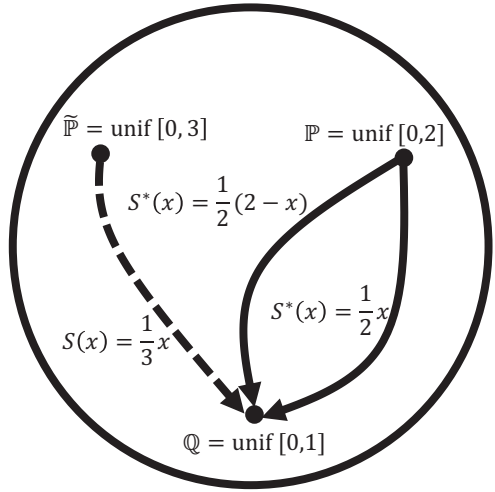}}
\caption{Transformation from a distribution $P=\text{unif}[0,2]$
to $Q=\text{unif}[0,1]$. There are multiple maps $S^*$ that push
forward $P$ to $Q$: $\frac{1}{2}x$ and $\frac{1}{2}(2-x)$. An
arbitrary map $S$ will always push forward \itt{some} $\tP$ to the
target $Q$, as shown above, but the $\tP$ does not need to be the
$P$ we started from.} \label{fig:image-distributions-maps2}
\end{figure}

Given a fixed density $q$ and an arbitrary diffeomorphism $S$, the
corresponding Jacobian equation for the induced density $\tp_S$ is
given by
\begin{eqnarray}
\tp_{S}(x) = q(S(x))|\det \parenth{J_S(x)}|, \quad \forall x \in
\cX. \label{eqn:JacobianEquation:b}
\end{eqnarray}
We denote it by $\tp_S$ to make clear that given $q$ and $S$, the
associated $\tp$, which $S$ pushes to $Q$, is functionally
dependent upon $S$ through the right hand side of
\eqref{eqn:JacobianEquation:b}. Note that the left-hand side of
\eqref{eqn:defn:JacobianEqn} involves the true $p$ induced by the
optimal $S^*$, whereas the left-hand side of
\eqref{eqn:JacobianEquation:b} involves the $\tp_S$ induced by an
arbitrary $S$.

Next, we propose an algorithm to find an optimal nonlinear map
that satisfies the Jacobian equation in
(\ref{eqn:defn:JacobianEqn}) by searching over all possible maps
$S$ that push forward some $\tP_S$ to $Q$ given by
\eqref{eqn:JacobianEquation:b}. Finding a map $S^*$ that pushes
$P$ to $Q$ is equivalent to finding a map $S$ for which the
``distance'' between $\tP_S$ and $P$ is zero.  Using KL
divergence, this becomes:
\begin{align}
\!\!\!D(P \| \tP_S) &= \E_P \brackets{\log \frac{p(X)}{\tilde{P}_S(X)}} \\
                    &= -h(P) + \E_P \brackets{-\log \tp_S(X)}
\end{align}
where $h(P)$ is the Shannon differential entropy of
$P$\cite{cover2006elements}, which is fixed with respect to $S$.
For the remainder of the manuscript, we assume the following:
\begin{assumption} \label{assumption:entropyP:finite}
$\E\brackets{|\log p(X)|} < \infty$.
\end{assumption}
From Jensen's inequality, Assumption~\ref{assumption:entropyP:finite} implies  $\abs{h(P)} < \infty$ and so we have:\\
\fbox{
\begin{tabular}{l l}
\problemA & $\displaystyle \min_{S \in \diffeomorphisms}$
$\kldist{P}{\tP_S}$ $\;\Leftrightarrow\;$  $\displaystyle \max_{S
\in \diffeomorphisms} \;\; \E_P \brackets{\log \tp_S(X)}$
\end{tabular}
}\\
$\;$\\
We note the following:
\begin{lemma}
$S$ is optimal for problem~\problemA if and only if $S$ pushes $P$
to $Q$.
\end{lemma}
\begin{proof}
By the definition of the Jacobian equation induced by $S$ in
\eqref{eqn:JacobianEquation:b}, $S$ pushes some $\tP_S$ to $Q$.
Assume $\tP_S=P$. Then $\kldist{P}{\tP_S}=0$, and from the
non-negativity of KL divergence, we have that $S$ solves
Problem~\problemA. Now, assume $S$ is optimal for
Problem~\problemA. Note that the KL divergence $\kldist{P}{\tP_S}$
is zero if and only if $P=\tP_S$ and thus $S$ pushes $\tP_S \equiv
P$ to $Q$.
\end{proof}

Now we note from Example~\ref{example:twoOptimalMaps} that in
general there can be more than one optimal solution, some of which
satisfy $\det\parenth{J_{S}(x)}>0$, and others, for which
$\det\parenth{J_{S}(x)} < 0$. Both maps are equally ``as good''.
However, from an optimization viewpoint, the search space is so
``rich''
 that this leads to non-convexity in an optimization problem.

\subsection{A Convex Problem in Infinite Dimensions}
\label{subsec:convexprob}

We here consider restricting our search to orientation-preserving
diffeomorphisms, i.e., ones with positive definite Jacobian: \beqa
\MD \triangleq \braces{S \in \diffeomorphisms: J_S(x) \succ 0 \;
\forall x \in \cX}.
\label{eqn:defn:orientation-preserving-diffeomorphisms} \eeqa This
eliminates possible maps such as $S^*(x)=\frac{1}{2}(2-x)$ in
Example~\ref{example:twoOptimalMaps}. As such, for any $S \in
\MD$, the Jacobian equation becomes:
\begin{eqnarray}
\tp_S(x) = q(S(x)) \det \parenth{J_S(x)}.
\label{eqn:ModJacobianEqn}
\end{eqnarray}
  This gives rise to following problem, which simply involves a restriction of the feasible set in Problem~\problemA to orientation-preserving maps:\\
\fbox{
\begin{tabular}{l l}
\problemB & $\displaystyle \min_{S \in \MD}$  $\kldist{P}{\tP_S}$
$\;\Leftrightarrow\;$  $\displaystyle \max_{S \in \MD} \;\; \E_P
\brackets{\log \tp_S(X)}$
\end{tabular}
}\\

With this, we can state the following theorem:
\begin{theorem} \label{theorem:problemB:optimal}
Problem \problemB has an optimal solution $S^*$, which is also
optimal for Problem~\problemA, and thus pushes $P$ to $Q$.
\end{theorem}
\begin{proof}
Since $P$ and $Q$ have densities $p$ and $q$ with respect to the
Lebesgue measure, we consider the Monge-Kontarovich problem with
Euclidean distance cost \cite{villani2003topics}:
\begin{eqnarray*}
\quad \min_{S: \cX \to \cX} && \int_{x \in \cX}  \|x
-S(x)\|^2 p(x)dx \\
\text{s.t.} && \pushforward{S}{P}{Q}.\nonumber
\end{eqnarray*}
Key properties of its optimal solution $S^*$ include: (i) $S \in
\diffeomorphisms$, and (ii) $S^*(x) = \nabla h(x)$ where $h$ is a
strictly convex function (which implies that $J_{S^*}(x) \succ 0$
for all $x \in \cX$). Thus $S^*$ lies in the feasible set of
problem~\problemB.  But from the Monge-Kantorovich problem, $S^*$
pushes $P$ to $Q$ and is thus optimal for Problem~\problemA. Since
the feasible set of problem~\problemA contains the feasible set of
problem~\problemB, $S^*$ is optimal for Problem~\problemB
\end{proof}

In Example~\ref{example:twoOptimalMaps}, both $\frac{1}{2}x$ and
$\frac{1}{2}(2-x)$ push $P$ to $Q$, but $\frac{1}{2}x$ is
increasing and the other is decreasing. Problem \problemB finds  a
map $S$ whose Jacobian matrix $J_S$ is positive definite (e.g.,
$S(x)=\frac{1}{2}x$). Remarkably, from
Theorem~\ref{theorem:problemB:optimal}, the restriction to
monotonicity suffers no loss in optimality.  Moreover, for many
problems of interest, it guarantees that finding the optimal
solution is tractable:
\begin{theorem}\label{thrm:loq-concavity}
  If $\log q(x)$ is concave in $x$, then problem \problemB is a convex optimization problem with a unique optimal solution.
\end{theorem}

\begin{proof}
Note that if $S,\tS \in \MD$  then $J_S,J_{\tS}\succ 0$. For any
$\lambda \in [0,1]$, $J_{\lambda S + (1-\lambda) \tS}(x)=\lambda
J_S(x) + (1-\lambda) J_{\tS}(x) \succ 0$.  Thus $\MD$ is convex.
Now we note from \eqref{eqn:ModJacobianEqn} that
\[\E_P \brackets{\log \tp_S(X)} = \int_{\cX} \brackets{\log q(S(x))+\log \det (J_S(x))} p(x) dx.\]
Since  $\log\det(\cdot)$ is strictly concave over the space of
positive definite matrices and since $\log q(\cdot)$ is concave,
we have that $\E_P \brackets{\log \tp_S(X)}$ is a sum of strictly
convex and convex functions, which itself is {\it strictly convex}
in $S$.  Thus an optimal solution $S^*$, which exists from
Theorem~\ref{theorem:problemB:optimal}, is unique.
\end{proof}

We make the following remark as it relates to other methods
involving KL divergence minimization:
\begin{remark}
Variational Bayes \cite{jaakkola2001tutorial} and EP
\cite{minka2001expectation} methods are also based on the
minimization of a KL divergence. However, the KL divergence
minimization is of a reverse kind, is not over a space of maps,
and is thus conceptually different.  Moreover, these methods build
upon deterministic approximations to the posterior,  do not
guarantee exactness, and in  general are non-convex.
\end{remark}

Although problem~\problemB is convex for log-concave $q$, note
that the feasible set is an infinite-dimensional space of
functions, and the objective function involves an expectation
(e.g. a $d$-dimensional integral) with respect to $P$. Both of
these require further effort to be implemented in real
computational settings.


\subsection{A Convex Problem in Finite Dimensions via Truncated Basis}
Problem \problemB involves minimizing an expectation with respect
to $P$ over $\MD$, an infinite-dimensional space of functions. We
here consider representing any $S \in \MD$ in terms of its
(truncated) orthogonal basis representation with respect to $P$.
We consider maps $m$ of the form:
\begin{eqnarray}
m(x) = \sum_{j \in \mathcal{J}}w_j \basisfn{j}(x)
\label{eqn:defn:linearbasis}
\end{eqnarray}
where $\basisfn{j}(x) \in \reals$ are $d$-variate bases, $w_j \in
\reals^d$ are basis coefficients, and $\mathcal{J}$ is a set of
all possible indices $j$.

One natural way to do this, if $\cX \subset \reals$, is to perform
a polynomial chaos expansion (PCE) of the nonlinear optimal map
\cite{xiu2003wiener, ernstConvergencePolynomialChaos2012}, meaning
that we select $(\basisfn{j}: j \geq 1)$ so that they are
orthogonal {\it with respect to $P$}: \beqa \int_{\cX}
\basisfn{i}(x) \basisfn{j}(x) p(x)dx = \indicatorvbl{i=j}.
\label{eqn:PCE:orthogonal} \eeqa For example, if $\cX = [-1,1]$
and $P$ is uniformly distributed, then $\basisfn{j}(x)$ are the
Legendre polynomials, and if $\cX = \reals$ and $P$ is Gaussian,
then $\basisfn{j}(x)$ are the Hermite polynomials
\cite{xiu2003wiener}.

\begin{remark}
  In principle, any basis of polynomials, for which the truncated expansion of
  functions is dense in the space of all functions on $\cX$, suffices. Using the
  PCE where orthogonality is measured with respect to the
  prior, means that computing conditional expectations and other calculations can
  be done only with linear algebra.
\end{remark}

When $|\mathcal{J}|=K$, we represent a nonlinear map as
\begin{eqnarray}
m(x)=W \Phi(x), \qquad\qquad d \times 1
\label{eqn:polynomialchoas}
\end{eqnarray}
where $W=[w_1,\ldots, w_K]$ is $d \times K$, and
$\Phi(x)=[\basisfn{1}(x),\ldots,\basisfn{K}(x)]^T$ is $K \times
1$. The Jacobian matrix $J_m(x)$ is also expressed as
\begin{eqnarray}
J_m(x) = W J_{\Phi}(x),  \qquad\qquad d \times d
\end{eqnarray}
where $J_{\Phi}(x)= \brackets{\partial \basisfn{i}/\partial x_j
(x)}_{i,j}$ is $K \times d$.

For $\cX \subset \reals^d$, a polynomial basis $\phi^{(j)}(x)$ can
be represented using tensor products as
\begin{eqnarray}
  \basisfn{j}(x) = \prod_{a=1}^d \psi_{j_a}(x_a)
\end{eqnarray}
where $\psi_{j_a}$ is a univariate polynomial of order $j_a$, and
$j \to (j_1,\ldots, j_d)$ is defined in a standard diagonal
manner. For example, if $d=2$, then we have $j \to (j_1,j_2)$
where $0\leq j_1+j_2\leq d$, and thus polynomial bases are given
by
\begin{eqnarray}
\begin{array}{cc}
  \basisfn{0}(x) = \psi_0(x_1) \psi_0(x_2), & \basisfn{1}(x) = \psi_0(x_1) \psi_1(x_2), \\
  \basisfn{2}(x) = \psi_1(x_1) \psi_0(x_2), & \basisfn{3}(x) = \psi_0(x_1) \psi_2(x_2), \\
  \basisfn{4}(x) = \psi_1(x_1) \psi_1(x_2), & \basisfn{5}(x) = \psi_2(x_1) \psi_0(x_2).
\end{array}
\end{eqnarray}


Since $J_m(x)=W J_{\Phi}(x)$ need not be positive definite, we use
the Euclidean projection, or equivalently the proximal operator of
the indicator function of $\MD$ \cite{ParikhBoydProximalAlgs2013}:
\beqa
S_W &=& \projectionMD{W \Phi} \\
 \projectionMD{W \Phi}(x) &\triangleq& \!\!\!\! \argmin_{m(x)=\tilde{W} \Phi(x): J_m(x) \succ 0}  \|m(x)-W\Phi(x)\|^2. \nonumber
\eeqa

By defining $\MDK \triangleq \braces{ S_W: W \in \reals^{d \times K} }$, we can define the following optimization problem:\\
\fbox{
\begin{tabular}{l l}
\problemCK & $\displaystyle \min_{S \in \MDK}$ $\kldist{P}{\tP_S}$
$\;\Leftrightarrow\;$  $\displaystyle \max_{S \in \MDK} \;\; \E_P
\brackets{\log \tp_S(X)}$
\end{tabular}
}\\

We now show that $\MDK \subset \MD$ is dense in $\MD$:

\begin{theorem}
If $q$ is log-concave, then problem~\problemCK is a finite
dimensional convex optimization problem with a unique optimal
solution, which we denote as $S_K$.  Moreover: \beqa \lim_{K \to
\infty} \kldist{P}{\tP_{S_K}} = 0. \eeqa
\end{theorem}

\begin{proof}
Define  $\Sinf$ to be the unique solution to Problem~\problemB. As
such, the random variable $Z=\Sinf(X)$ is drawn according to $Q$,
which is log-concave. It is well known that any log-concave random
variable satisfies the sub-exponential tail property:
\begin{eqnarray*}
\E \brackets{ e^{c \norm{Z}}} < \infty, \quad \text{ for some } c
> 0.
\end{eqnarray*}
This is a sufficient condition \cite[Theorem
3.7]{ernstConvergencePolynomialChaos2012} for the tensor products
of the generalized polynomial chaos  $\Phi_1(x), \Phi_2(x),
\ldots$ to be dense in $L_2(\cX,\sigma(X),P)$, implying:
\begin{align}
Z = \Sinf(X) &\underbrace{=}_{L_2} \lim_{K \to \infty} v_K(X), \label{eqn:Sinf:Ltwobasisexpansion} \\
v_K(X)  &\triangleq \sum_{i=1}^K \innerprod{\Sinf(X),\Phi_i(X)}
\Phi_i(X).
\end{align}
Now define \beqa
 \tS_K(X)= \projectionMD{v_K}(X), \label{eqn:defn:tSK}
\eeqa and note that \beqa
&&  \lim_{K \to \infty} \E\brackets{\|\Sinf(X)-\tilde{S}_K(X)\|^2} \nonumber \\
&=& \lim_{K \to \infty} \E\brackets{\|\projectionMD{\Sinf}(X)-\projectionMD{v_K}(X)\|^2} \label{eqn:PCE:convergence:optimalMap:b}\\
&\leq& \lim_{K \to \infty} \E\brackets{\|\Sinf(X)-v_K(X)\|^2} \label{eqn:PCE:convergence:optimalMap:c}  \\
&=&  0 \label{eqn:PCE:convergence:optimalMap:d} \eeqa where
\eqref{eqn:PCE:convergence:optimalMap:b} follows because
$\Sinf(X)=\projectionMD{\Sinf}(X)$ (since $\Sinf \in \MD$) and
from \eqref{eqn:defn:tSK};
\eqref{eqn:PCE:convergence:optimalMap:c} follows from the  firm
non-expansive property of the proximal operator
\cite{ParikhBoydProximalAlgs2013}; and
\eqref{eqn:PCE:convergence:optimalMap:d} follows from
\eqref{eqn:Sinf:Ltwobasisexpansion}.

Thus from  \eqref{eqn:PCE:convergence:optimalMap:d}, we have that
$\tilde{S}_K(X) \to_{L_2} \Sinf(X)$.  Note that $\log
q(\Sinf(X))+\log\det J_{\Sinf}(X) = \log p(X)$ which from
Assumption~\ref{assumption:entropyP:finite} lies in
$L_1(\cX,\sigma(X),P)$.  In addition, $\log q(\cdot)$ is concave
(and thus continuous), and  $\log \det(\cdot)$ is concave (and
thus continuous) over the space of positive definite matrices.
Since $L_2$ convergence implies $L_1$ convergence, we have:
\begin{eqnarray*}
\log \frac{p(X)}{\tp_{\tS_K}(X)} \!\!\!\!\!\!&=& \!\!\!\! \log p(X) - \log q\parenth{\tS_K(X) } -\log \det J_{\tS_K}(X) \\
 \!\!\!\!\!\!\!\!&\underbrace{\to}_{L_1}& \!\!\!\! \log p(X) - \log q\parenth{\Sinf(X)} -\log \det J_{\Sinf}(X) \\
 &=& 0.
\end{eqnarray*}
Since $L_1$ convergence implies convergence in distribution, we
have that $\kldist{P}{\tP_{\tS_K}} \to 0$. But since $\tS_K \in
\MDK$ and $S_K$ is the optimal solution to problem~\problemCK, we
have that $\kldist{P}{\tP_{S_K}} \leq \kldist{P}{\tP_{\tS_K}} \to
0$.
\end{proof}

\subsection{Stochastic Convex Optimization in Finite Dimensions}
\label{subsec:implementation:finiteDimensions} Note that the
expectation in problem~\problemCK is given as
\[\E_P \brackets{\log \tp_S(X)} = \int_{\cX} \brackets{\log q(S(x))+\log \det (J_S(x))} p(x) dx,\]
which cannot in general be evaluated for any $S$.  But since this
is an expectation with respect to $P$, we can define a probability
space $(\Omega,\mathcal{F},\P)$ pertaining to $\textit{i.i.d.}$
samples $(X_1,X_2,\ldots)$ drawn from from $P$. We define $P_n$ as
the empirical distribution on $(X_1,\ldots,X_n)$ and then consider
the empirical expectation:
\[\E_{P_n} \brackets{\log \tp_S(X)} = \frac{1}{n}\sum_{i=1}^n \log q(S(X_i))+\log \det (J_S(X_i)).\]
This gives rise to a  stochastic optimization problem \\
\fbox{
\begin{tabular}{l l}
\problemDKn & $\displaystyle \max_{S \in \MDK} \;\; \E_{P_n}
\brackets{\log \tp_S(X)}$
\end{tabular}
}\\
Since this problem only involves $S(x)$ and $J_S(x)$ evaluated at
$(X_1,\ldots,X_n)$, we can solve this with a multi-step procedure:
given \textit{i.i.d.} $\parenth{X_i}_{i=1,N}$ from $P$, evaluate a
priori the associated $\parenth{\Phi_i \triangleq
\Phi(X_i)}_{i=1,N}$,  $\parenth{J_i \triangleq
J_\Phi(X_i)}_{i=1,N}$. From here, we solve for
\begin{subequations} \label{eqn:W:maximization}
\beqa
\max_{W \in \reals^{d\times K}}  && \frac{1}{N}\sum_{i=1}^N \log q\parenth{W \Phi_i}+\log \det (W J_i) \\
 s.t. && W J_1 \succ 0, \ldots, W J_N \succ 0.
\eeqa
\end{subequations}
Given $\WKn$, we employ the proximal operator
 \beqa
 \SKn(x)  \triangleq S_{\WKn}(x) = \projectionMD{\WKn \Phi}(x). \label{eqn:defn:SKn}
 \eeqa
 so that $\SKn \in \MDK$.

\begin{remark}
Note that Problem~\problemDKn is only tractable if generating
\textit{i.i.d.} samples from $P$ is tractable.  Luckily, $P$ in
many situations is a well-defined distribution that can be sampled
from easily, such as Gaussian, exponential family, uniform, or sum
of Gaussians.    More generally, if $p(x)$ is log-concave, then we
can do as follows: simply first solve an auxiliary
Problem~$\overline{\text{\problemDKn}}$ with
$\bar{P}=\text{uniform}$ or $\bar{P}=\mathcal{N}(0,\Sigma)$, or
any other distribution we can easily sample from.  Then define
$\bar{Q}=P$.  By implementing problem
$\overline{\text{\problemDKn}}$, we will generate a map $\bar{S}$
that can transform \textit{i.i.d.} samples from $\bar{P}$, which
is easy to sample from, into \textit{i.i.d.} samples from $P$.
From here, we can move forward and implement Problem $\problemDKn$
to generate a map $S$ to push $P$ to $Q$.
\end{remark}

This optimization problem can be implemented with convex
optimization software such as CVX in MATLAB, CVXPY in Python, etc.
Here, we implemented the algorithms with CVX \cite{grant2008cvx}
in MATLAB.

\begin{theorem}\label{theorem:problemCKn:KLminimization}
The map $\SKn$ as defined in \eqref{eqn:defn:SKn} is the unique
minimizer of problem~\problemDKn and moreover, \beqa \lim_{K \to
\infty} \limn \kldist{P}{\tP_{\SKn}}  = 0  \quad \P-a.s.
\label{eqn:thm:problemCKn:KLminimization} \eeqa
\end{theorem}
\begin{proof}
The uniqueness of the minimizer follows from strict concavity of
the problem.  That $\SKn$ is the minimizer follows because
\[ \SKn(X_i) = \projectionMD{\WKn \Phi}(X_i) = \WKn\Phi(X_i), \;\; i=1,\ldots,n \]
because the positive definite constraint was imposed on $\WKn$ in
the feasible set of the problem in \eqref{eqn:W:maximization}.
Thus $\SKn$ corresponds to an optimal solution of the optimization
problem defined in \eqref{eqn:W:maximization}, which is a
relaxation to problem~\problemDKn: they have the same objective
but the feasible set of the problem in  \eqref{eqn:W:maximization}
contains the feasible set in \problemDKn. Since $\SKn$ solves the
relaxation, it solves the original problem.

Now, we can prove \eqref{eqn:thm:problemCKn:KLminimization}. Note
that for any $S=\projectionMD{W\Phi} \in \MDK$ and
$S'=\projectionMD{W'\Phi} \in \MDK$, we have \beqa
\|S-S'\|^2_P &=& \int_{\cX} \|\projectionMD{W\Phi}(x)-\projectionMD{W'\Phi}(x)\|^2p(x)dx \nonumber \\
             &\leq& \int_{\cX} \|W\Phi(x)-W'\Phi(x)\|^2 p(x)dx \label{eqn:proof:thm:problemCKn:KLminimization:a}\\
             &=& \int_{\cX} \|(W-W') \Phi(x)\|^2 p(x)dx \nonumber\\
             &=& \int_{\cX} \Phi(x)^T (W-W')^T(W-W') \Phi(x) p(x)dx \nonumber\\
             &=& \|W-W'\|_P^2 \label{eqn:proof:thm:problemCKn:KLminimization:b}
\eeqa where \eqref{eqn:proof:thm:problemCKn:KLminimization:a}
follows from the firm non-expansive property of the proximal
operator \cite{ParikhBoydProximalAlgs2013}; and
\eqref{eqn:proof:thm:problemCKn:KLminimization:b} follows from the
fact that $\Phi$ is a $P$-orthonormal matrix.
Now since $W,W' \in \reals^{d \times K}$, from
\eqref{eqn:proof:thm:problemCKn:KLminimization:b}, we have that
with $\MDK$ is locally compact. Thus we can exploit the strict
convexity of problem~\problemDKn over a locally compact constraint
set and Assumption~\ref{assumption:entropyP:finite} to conclude
\cite[Theorem 3.1]{Berk72ConsistencyNormality} that $\prob{\lim_{n
\to \infty}\SKn = \SK}=1$.


With this, we have that
\[
\limn \kldist{P}{\tP_{\SKn}}  = \kldist{P}{\tP_{\SK}} \quad
\P-a.s.
\]
By taking an outer limit in $K$, we complete the proof.
\end{proof}



\section{Optimal Transport and Bayesian Inference}
\label{sec:BayesInference}

In this section, we apply the general push-forward theorem to the
context of Bayesian inference. We demonstrate that for a large
class of prior and likelihood families, an optimal map can be
efficiently constructed through convex optimization.

We assume that $X \in \reals^d$ is drawn \itt{a priori} according
to $P_X$, which has a density $p_X(x)$ with respect to Lebesgue
measure. We define the likelihood function in density form as
$p_{Y|X}(y|x)$. Having observed $Y=y$, the posterior distribution
$P_{X|Y=y}$ has a density given by $p_{X|Y=y}(x|y)$, determined by
Bayes' rule in \eqref{eqn:BayesRule}.
In general, the calculation of $\beta_y$ is intractable,
especially when $d$ is in high dimension or $p_X(x)$ is not a
conjugate prior for $p_{Y|X}(y|x)$. Several methods have been
developed to bypass its calculation or approximate different
functions of the posterior. Typical approaches to perform this are
Monte Carlo methods \cite{Liu2008}. MCMC methods are families
where samples are drawn from a Markov chain, whose invariant
distribution is that of the posterior. As described in
Section~\ref{sec:intro}, one problem of these methods is that
because samples are drawn from a Markov chain, they are
necessarily statistically {\bf dependent}; so the law of averages
kicks in more slowly.  Also, efficient MCMC methods are typically
tailored to the specifics of the prior and likelihood and as such,
lack generality. In addition, many natural situations
\itt{require} the calculation of $\beta_y$, such as information
gain calculations mentioned in
\eqref{eqn:posteriorInformationGain}.


\subsection{Problem Setup}
\label{subsec:problemsetup} To make use of the general
push-forward theorem, we set $P=P_X$ as the prior distribution
$Q=P_{X|Y=y}$ as the posterior distribution. Then we can find a
diffeomorphism $S_y^*$, for which $S_y^* \# P = Q$, or
equivalently $S_y^* \# P_X = P_{X|Y=y}$. In this Bayesian
inference framework, we denote the optimal map $S^*$ as $S_y^*$
with a subscript $y$, since a map depends on each observation $y$.
The Jacobian equation for orientation-preserving maps
\eqref{eqn:ModJacobianEqn} within the context of Bayes' rule
\eqref{eqn:BayesRule} becomes
\begin{eqnarray}
p_X(x) = \frac{p_{Y|X}(y|S^*_y(x)) p_X\parenth{S^*_y(x)}}{\beta_y}
\det{\parenth{\jac_{S^*_y}(x)}}.
\label{eqn:JacobianEqn:BayesRule:a}
\end{eqnarray}
Next, we (a) exchange $p_X(x)$ in the left-hand side of
(\ref{eqn:JacobianEqn:BayesRule:a}) with $\beta_y$ in the
right-hand side in order to put all the $x$-dependent terms to the
right-hand side and (b) take the logarithm and use an arbitrary
map $S_y(x)$, to define the operator $T: \MD \times \cX \to
\reals$ as in \cite{el2012bayesian}:
\begin{align}
T(S,x) &\triangleq \log p_{Y|X}(y|S(x)) + \log p_X\parenth{S(x)} \nonumber \\
&+  \logdetJSx - \log p_X(x). \label{eqn:defn:F}
\end{align}
Using (\ref{eqn:defn:F}), we can now state
\eqref{eqn:JacobianEqn:BayesRule:a} equivalently as follows.
\begin{lemma}
  \label{lemma:variationalPrinciple:a}
  A diffeomorphism $S^*_y$ satisfies $\pushforward{S_y^*}{P_X}{P_{X|Y=y}}$ if and
  only if
  \beqa
    T(S_y^*,x) \equiv \log \beta_y, \quad \forall x \in \cX.
    \label{eqn:variationalPrinciple:a}
  \eeqa
\end{lemma}
Note that this encodes a variational principle. The left-hand side
of (\ref{eqn:variationalPrinciple:a}) is allowed to vary with $x$.
But for $S^*_y$, at {\bf any} $x$, $T(S_y^*,x)$ takes on the same
value. Thus this suggests a particular problem
formulation\cite{el2012bayesian}:
\begin{align}
  \qquad S_y^* &= \argmin_{S_y \in \MD} \obj_1(S_y), \\
  \obj_1(S_y) &\triangleq \int_{x \in \cX}\parenth{T(S_y,x)-\E[T(S_y,X)]}^2 p_X(x)
  dx. \nonumber
\end{align}

\begin{remark}
  Attempting to solve this problem is in general computationally intractable.
  If we approximate a diffeomorphism decision variable $S_y$ using a truncated basis expansion (e.g. PCE),
  the above problem is still non-convex for log-concave priors and likelihoods. This follows because
  under log concavity, $T(S,X)$ is concave in $S$, but quadratic functions of differences of concave functions are in general non-convex.
\end{remark}

  Our insight is to
  abandon the approach espoused in \cite{el2012bayesian} and instead focus on a
  subset of common problems with log-concave structure, using an alternative KL
  divergence based criterion. This leads to a computationally tractable algorithm via convex
  optimization.

\subsection{A convex problem}
We now show that for many natural priors and likelihoods we can
efficiently find a diffeomorphism $S_y^*$, for which
$\pushforward{S_y^*}{P_X}{P_{X|Y=y}}$, using an alternative
optimality criterion described in Section~\ref{sec:gen-push-thrm}.
Consider any other diffeomorphism $S_y(x)$ that induces some
$\tilde{P}_X$ whose density is denoted as $\tp_X(x)$. From the
modified Jacobian equation~(\ref{eqn:JacobianEqn:BayesRule:a}):

\begin{eqnarray}
\tp_X(x) = \frac{p_{Y|X}(y|S_y(x)) p_X\parenth{S_y(x)}}{\beta_y}
\det{\parenth{\jac_{S_y}(x)}}. \label{eqn:defn:tf:a}
\end{eqnarray}
By careful inspection of \eqref{eqn:defn:tf:a} and
\eqref{eqn:defn:F}, we then have that

\begin{eqnarray}
\log \frac{p_X(x)}{\tp_X(x)} = \log \beta_y - T(S_y,x).
\label{eqn:LLLR-T:a}
\end{eqnarray}
So if a diffeomorphism $S_y^*$ satisfies $S_y^* \# P_X =
P_{X|Y=y}$, then $p_X(x)=\tp_X(x)$, which means
\begin{eqnarray}
0 = \log \beta_y - T(S^*_y,x) \Leftrightarrow T(S^*_y,x)=\log
\beta_y, \;\; \forall x \in \cX. \label{eqn:LLLR-T:b}
\end{eqnarray}
Thus we find a map to minimize the KL-divergence between $P_X$ and
$\tilde{P}_X$, which is given by
\begin{eqnarray}
\kldist{P_X}{\tilde{P}_X} =\log \beta_y - \int_{x \in \cX} p_X(x)
T(S_y,x) dx. \nonumber
\end{eqnarray}
This suggests the following optimization problem, equivalent to
\problemB in Section~\ref{sec:gen-push-thrm} for $P=P_X$ and
$Q=P_{X|Y=y}$:
\begin{eqnarray}
\qquad S_y^* &=& \argmax_{S_y \in \diffeomorphisms{X}} \int_{x \in
\cX} p_X(x) T(S_y,x) dx. \label{eqn:objectivefn:a}
\end{eqnarray}
Also note that once we have solved for $S_y^*$ in
\eqref{eqn:objectivefn:a}, we in addition can obtain $\beta_y$ by
virtue of evalution of the $T$ operator using any $x \in \cX$ in
\eqref{eqn:variationalPrinciple:a}.
This is fundamentally the optimization problem we aim to solve for
Bayesian inference. By phrasing the inference problem as an
optimal transport problem along with the natural assumption of
log-concavity of the prior and likelihood, we can create a
computationally efficient method to carry out Bayesian inference
via convex optimization.

\subsection{Implementation}
\label{subsec:implementation} Applying the PCE in
\eqref{eqn:defn:linearbasis} to approximate the function $S(x)$,
we re-define $T(S,x)$ in \eqref{eqn:defn:F} as
\begin{align}
\tilde{T}(W,x) &\triangleq \log p_{Y|X}(y|W\Phi(x)) + \log p_X\parenth{W \Phi(x) } \nonumber \\
& + \log \det (W J_A(x)) - \log p_X(x).
\label{eqn:defn:T:linearbasisexpansion}
\end{align}
By truncating the PCE and approximating the expectation by a
weighted sum of \textit{i.i.d.} samples, we finally arrive at the
computationally tractable convex inference problem for Bayesian
inference:

\begin{align}
\qquad W^* &= \argmax_{W \in \reals^{d\times K} }
\frac{1}{N}\sum_{i=1}^N [\log p_{Y|X}(y|W \Phi(X_i) )\nonumber\\
&\quad +\log p_X\parenth{W \Phi(X_i)} + \log \det (W J_A(X_i))]\nonumber \\
&\text{s.t.}\;\; WJ_{A}(X_1) \succ 0, \ldots, WJ_{A}(X_N) \succ 0
\label{eqn:objectivefn:c}
\end{align}
where $X_1, X_2,\ldots,X_N$ are \textit{i.i.d.} samples drawn from
$P_X$.

\begin{lemma}
\label{lemma:log-concavity-implies-T-concave} If $p_X(x)$ is
log-concave and $p_{Y|X}(y|x)$ is log-concave in $x$, then
$p_{X|Y=y}(x|y)$ is log-concave and the problem
\eqref{eqn:objectivefn:c} is a convex optimization problem.
\end{lemma}

\begin{proof}
That $p_{X|Y=y}(x|y)$ is log-concave in $x$ is trivial: it follows
directly from the assumption that $p_X(x)$ and $p_{X|Y=y}(x|y)$
are log-concave in $x$, along with Bayes' rule
\eqref{eqn:BayesRule}. As for showing that $\tilde{T}(W,x)$ is
concave in $W$, this follows from (i) the assumption that $p_X(x)$
and $p_{X|Y=y}(x|y)$ are log-concave in $x$; and (ii) that
concavity is preserved under affine transformations ($W \to
W\Phi(x)$, $W \to W J_A(x)$)\cite{boyd2004convex}. As for the
feasible set, a set of vectors satisfying an affine positive
definite constraint is convex \cite{boyd2004convex}.
\end{proof}

\section{Results}
\label{sec:results}

This section demonstrates how the proposed method was applied in
several examples. Firstly, we tested the accuracy of the proposed
method using a conjugate distribution where  the closed-form
expression of the posterior is known. Next, we considered several
Bayesian inference problems using simulated and real data.
Table~\ref{tab:summary} summarizes the basic settings of these
problems in terms of observations, unknown parameters,
likelihoods, and priors. As described, we considered Gaussian
likelihoods with a sparse (Lapacian or exponential) prior; and
logistic regression likelihoods with a Gaussian prior.  We
considered many fully Bayesian scenarios, including sampling from
the posterior, estimating Bayesian credible regions, and risk
minimization.  We compare performance to point estimation
counterparts by comparing expected losses and by comparing
receiver operating characteristic (ROC) curves. For the sake of
brevity, we denote the densities, $p_{X}(x)$, $p_{Y|X}(y|x)$, and
$p_{X|Y=y}(x|y)$ as $p(x)$, $p(y|x)$, and $p(x|y)$, respectively.

\begin{table}
\caption{Summary of problem settings in Section~\ref{sec:results}:
1) Gaussian likelihood with sparse prior (Laplace or exponential
for $\mathbf{x} > 0$); and 2)
Logistic regression with Gaussian prior. 
} \centerline{
\begin{tabular}{|c|c|c|}
  \hline
  Settings & Gauss $\times$ Laplace & Logis Reg $\times$ Gauss \\
  \hline
  \hline
  $\mathbf{Y}$, $\mathbf{X}$ & $\mathbb{R}^d$, $\mathbb{R}^d$ & $\mathbb{R}^d$, $\mathbb{R}^d$ \\
  \hline
  $-\log P_{\mathbf{Y}|\mathbf{X}}(\mathbf{y}|\mathbf{x})$ &  $\propto ||\mathbf{y}-\mathbf{Mx}||_2^2$ & $\log (1+e^{\mathbf{x}^{\text{T}}\mathbf{y}})-c\mathbf{x}^{\text{T}}\mathbf{y}$ \\
  \hline
  $-\log P_{\mathbf{X}}(\mathbf{x})$ & $\propto ||\mathbf{x}||_1$ & $\propto ||\mathbf{x}||_2^2$ \\
  \hline
\end{tabular}}\label{tab:summary}
\end{table}

\subsection{Conjugate distribution: comparison to closed-form}
\label{subsec:conjugateprior}

We demonstrate how accurately the proposed algorithm can construct
a map using the conjugate distribution where the posterior has the
same form as the prior and could be expressed as closed-form. In
this example, we chose a Poisson likelihood function (where
$\cY=\braces{0,1,\ldots}$) and its corresponding conjugate prior,
a Gamma distribution (where $\cX=[0,\infty)$).

When we have a Poisson likelihood expressed by $p(y|x) = x^y
e^{-x}/y!$ and its conjugate Gamma prior expressed by $p(x) =
1/(\Gamma(a)b^{a})x^{a-1}e^{-x/b}$, then the posterior, $p(x|y)$,
is computed as the Gamma distribution, which is given by

\begin{eqnarray}
p(x|y) \propto x^{a+y-1}e^{-(b+1)/b x}. \label{eqn:gammaposterior}
\end{eqnarray}
That is, the hyper-parameters in the Gamma prior, $a$ and $b$, are
changed as $a+y$ and $b/(b+1)$ in the Gamma posterior. In this
simulation, we set $a=2$ and $b=0.5$, and specified an observation
$y=1$. We then designed multiple maps to push forward the Gamma
prior to the Gamma posterior by changing the number of
\textit{i.i.d.} samples, $N$, drawn from the Gamma prior. We chose
the number of maximum order of the polynomial in
\eqref{eqn:defn:linearbasis} to be 5, resulting in $K=6$.

We then tested the accuracy of the constructed optimal map. The
solid line in Fig.~\ref{fig:accuracy} plots the variance of
$\tilde{T}(W,x)$ in the log-scale. For the optimal parameter
$W^{*}$, $\tilde{T}(W^{*},x)$ is almost constant over $x$, so the
variance is close
 to zero. We also computed the KL-divergence between the
original prior, $P_X$, and the map-dependent prior,
$\tilde{P}_{S^*}$, estimated using the designed map, shown by the
dotted line. As shown, the accuracy increases in both cases as we
increase the number of \textit{i.i.d.} samples we used for the
construction.

\begin{figure}[t]
\centerline{\includegraphics[width=8.0cm]{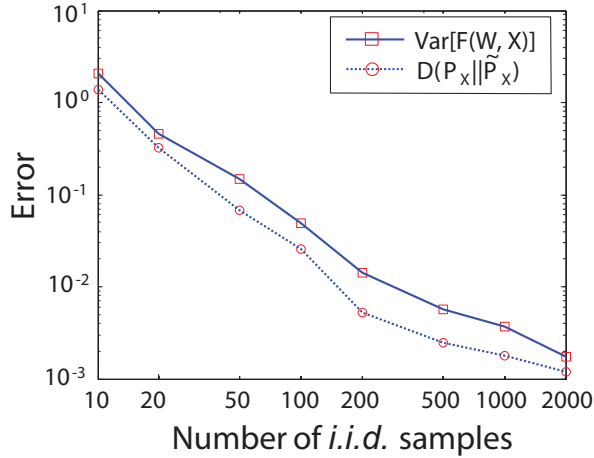}}
\caption{Conjugate distribution problem. The variance of
$\tilde{T}(W,x)$ and the KL-divergence between $P$ and
$\tilde{P}_{S^*}$ decreased with the number $n$ of \textit{i.i.d.}
samples drawn from the prior.} \label{fig:accuracy}
\end{figure}

Fig.~\ref{fig:mapconjugate} illustrates the actual transformation
of the Gamma prior to the Gamma posterior. The curve inside the
plot represents the designed optimal map, constructed using
$N=1000$ \textit{i.i.d.} samples. The curve on the $y$-axis
represents the true posterior, $p(x|y)$, in
(\ref{eqn:gammaposterior}). The histogram on the $y$-axis was
generated using the posterior samples that were obtained by
transforming the prior samples on the $x$-axis through the
designed map, $S_y^*(x)= \projectionMD{W^*\Phi}(x)$. The true
posterior, $p(x|y)$, matched well with the histogram of the
posterior samples obtained using the designed map.  This
demonstrated that the proposed method constructed the desired map
accurately. We also emphasize that it is straightforward to
generate samples from the posterior distribution by
\itt{transforming} samples drawn from the prior distribution -
which is usually \itt{easy} to sample from.

\begin{figure}[t]
\centerline{\includegraphics[width=8.0cm]{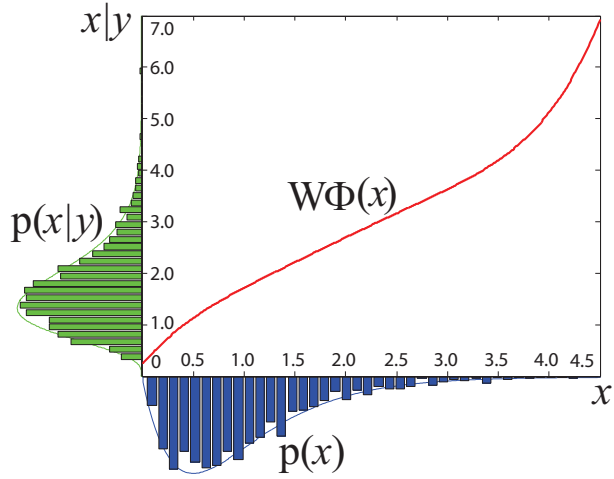}}
\caption{Illustration of the transformation of a prior, $p(x)$, to
a posterior, $p(x|y)$. The red curve represents a designed
nonlinear optimal map to push forward the prior on the $x$-axis to
the posterior on the $y$-axis.} \label{fig:mapconjugate}
\end{figure}

\subsection{Bayesian Credible Region}
\label{subsec:pointvsinterval}

In statistics, point estimates give a single value, which serves
as the \itt{best estimate} of an unknown parameter. For example,
we can calculate the mean of samples, or estimate a value to
maximize the likelihood or the posterior of the unknown parameter.
However, point estimates have several drawbacks, a major one of
which is that they do not provide any uncertainty measure of their
estimate. It is often important to know how reliable our estimate
is in many applications, and thus is desirable to calculate an
interval estimate (called Bayesian credible region), within which
we believe the unknown population parameter lies with high
probability.

The Bayesian credible region is not uniquely defined, and there
are several ways to define it: choosing the narrowest region
including the mode, choosing a central region where there exists
an equal mass in each tail, choosing a highest probability region,
all the points outside of which have a lower probability density,
etc. However, it is generally tricky to obtain these credible
regions in high dimensions, even though we can compute the
posterior distribution. In this paper, we introduce two approaches
to compute credible regions using the designed optimal map.

First, in order to compute the (1-$\alpha$) credible region of the
posterior where $0\leq \alpha \leq 1$, we obtain a region of the
prior, within which (1-$\alpha$) of the prior probability mass is
contained. We call this region the region with (1-$\alpha$)
confidence. The region of the prior with (1-$\alpha$) confidence
is generally easy to obtain; for instance, for a $d$-dimensional
Gaussian prior with zero mean and $\sigma^2$ variance, we choose
the region with (1-$\alpha$) confidence as the $d$-sphere with
$r_{\alpha}$ radius where $r_{\alpha}$ satisfies
$P(\|\mathbf{X}\|_2^2 \leq r_{\alpha}^2)=1-\alpha$. Since
$\|\mathbf{X}\|_2^2\sim \sigma^2\chi^{2}_d$ where $\chi^{2}_d$
represents the Chi-squared distribution with $d$ degrees of
freedom, we can compute $r_{\alpha}^2/\sigma^2$, above which
$\chi^{2}_d$ has its $\alpha$ probability mass. Then, it is
straightforward to compute the (1-$\alpha$) credible region of the
posterior distribution by transforming this $d$-sphere through the
designed optimal map.

To help illustrate this approach, we go through an example as
follows. Suppose that we have a binary class dataset as shown in
Fig.~\ref{fig:pointvsinterval} where a red plus sign represents
samples from one class, and a blue circle from the other. We
denote samples (or regressors) in the 2-dimensional space by
$\mathbf{y}_{i}=[y_{i,1},y_{i,2}]^{\text{T}}$, unknown parameters
of the logistic regression by $\mathbf{x}=[x_1,x_2]^{\text{T}}$,
and the class label corresponding to the $i$th sample
$\mathbf{y}_i$ by $c_i\in\{0,1\}$. Then, the logistic regression
likelihood function is given by
$p(\mathbf{y}|\mathbf{x})=\prod_{i}
p(c_i,\mathbf{y}_i|\mathbf{x})$ where a logistic regression model
is
$p(c=1,\mathbf{y}|\mathbf{x})=e^{{\mathbf{x}^{\text{T}}\mathbf{y}}}/(1+e^{{\mathbf{x}^{\text{T}}\mathbf{y}}})$
for $c=1$, and
$p(c=0,\mathbf{y}|\mathbf{x})=1-p(c=1,\mathbf{y}|\mathbf{x})$ for
$c_i=0$. We model a prior on $\mathbf{X}$ using a Gaussian
distribution with zero mean and 100 variance to regularize the
problem. Although both the prior and likelihood functions are
log-concave over $\cX$, there is no closed-form expression of the
posterior distribution; thus, we constructed an optimal map to
transform the prior to the posterior.

The big circle in Fig.~\ref{fig:credibleregion}~(a) illustrates
the region with 95 $\%$ confidence for Gaussian prior in 2-d
space. This region was obtained using the method described above.
Fig.~\ref{fig:credibleregion}~(b) illustrates the Bayesian
credible regions obtained by transforming the region with 95 $\%$
confidence of the prior in (a) through the designed optimal map.
The 95 $\%$ credible regions in Fig.~\ref{fig:pointvsinterval}~(c)
and (d) were also obtained in the same manner.

Secondly, we describe another approach to obtain Bayesian credible
regions using the \textit{i.i.d.} samples drawn from the
posterior. Once we design the optimal map, it is straightforward
to generate \textit{i.i.d.} samples drawn from the posterior by
transforming \textit{i.i.d.} samples drawn from the prior. For
example, the scatter plots in Fig.~\ref{fig:credibleregion}~(a)
and (b) show 2000 \textit{i.i.d.} samples drawn from the prior and
the posterior, respectively; The samples in (b) were generated by
transforming the samples in (a) through the optimal map. Then we
can find the credible interval for each parameter, within which
the (1-$\alpha$) portion of the samples is contained. The solid
vertical and horizontal lines in Fig.~\ref{fig:credibleregion}~(b)
represent 95 $\%$ central regions where there exist 5 $\%$ of
total samples (50 samples) in both tails.

\begin{figure}[t]
\centerline{\includegraphics[width=12cm]{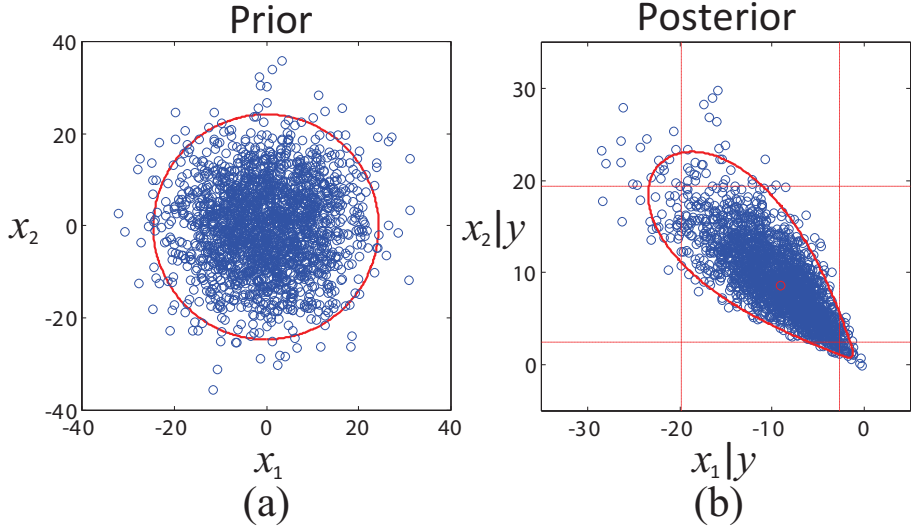}}
\caption{Transformation of confidence regions of prior to credible
regions of posterior using designed optimal map. (a) Region of
Gaussian prior with 95 $\%$ confidence is illustrated by the
circle. 2,000 \textit{i.i.d.} samples drawn from prior are
scatter-plotted together. (b) Bayesian credible regions with 95
$\%$ confidence is illustrated by fan shape. 2,000 \textit{i.i.d.}
samples drawn from posterior are scatter-plotted together. The
vertical and horizontal lines also represent the 95 $\%$ credible
region for the marginal distribution of each parameter.}
\label{fig:credibleregion}
\end{figure}

Next, using the proposed method we obtained Bayesian credible
regions of two datasets in binary classes as shown in
Fig.~\ref{fig:pointvsinterval}~(a) and (b), respectively. The
dataset in Fig.~\ref{fig:pointvsinterval}~(a) included 100 samples
from each class, and the dataset in
Fig.~\ref{fig:pointvsinterval}~(b) has 2 samples for each class.
Red plus signs represent samples belonging to one class, and blue
circles to the other. The MAP estimates of $x$ of both datasets
are same as illustrated as the dots at $(-3.9,3.9)$ in
Fig.~\ref{fig:pointvsinterval}~(c) and (d). Although the samples
in the two datasets had very different numbers and were
differently distributed, MAP estimation provided us with an
identical inference of the unknown parameters.

\begin{figure}[t]
\centerline{\includegraphics[width=10cm]{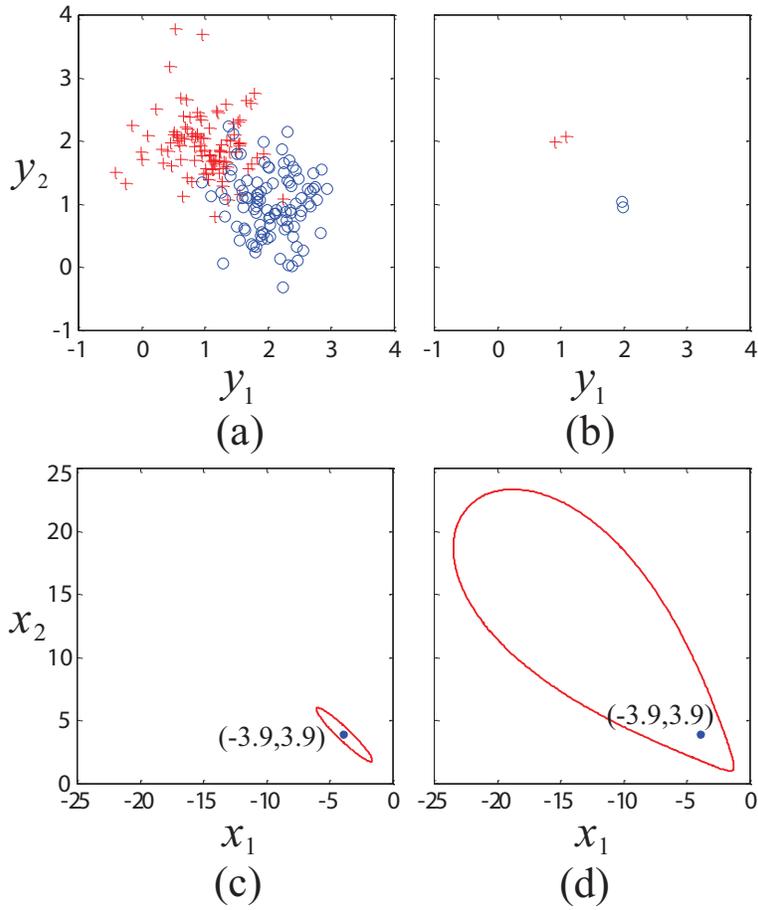}}
\caption{(a) Binary-class data with $200$ samples. (b)
Binary-class data with $4$ samples. (c) $95 \%$ Bayesian credible
region (within red line) with an MAP point estimate (blue point)
using the data in (a). (d) $95 \%$ Bayesian credible region with
an MAP point estimate using the data in (b).}
\label{fig:pointvsinterval}
\end{figure}

In addition to the point estimate, we obtained the Bayesian
credible region as a measure of confidence, as illustrated by the
red contours in Fig.~\ref{fig:pointvsinterval}~(c) and (d).
Although we obtained identical MAP estimates from both datasets,
we had very different credible regions. The dataset in
Fig.~\ref{fig:pointvsinterval}~(a) with 200 samples provided us
much smaller credible region than the dataset in
Fig.~\ref{fig:pointvsinterval}~(b) with 4 samples, meaning that we
were more confident on the estimate obtained using the dataset in
Fig.~\ref{fig:pointvsinterval}~(a). There exists more variability
in the direction of quadrant II in both datasets, and of course,
we can see much more variability for the dataset in
Fig.~\ref{fig:pointvsinterval}~(b). There is less variability in
the direction of quadrant IV, since the parameters in quadrant IV
switch the classification result.

\subsection{Bayes Risk}
\label{subsec:BayesMAPdecision} So far we have demonstrated how we
can compute the posterior probability of an unknown parameter and
use it to quantify the degree of \itt{uncertainty}. In other
situations, we perform Bayesian inference for the purpose of
taking some action. What action is taken can depend on the
estimate of some unknown parameter $x \in \cX$, our estimate of an
unknown label $c \in \mathsf{C}$ of an observation, etc. Here, we
demonstrate how an action can be improved when using the computed
posterior as compared to when using point estimates (e.g. MAP).

Suppose that we take some action $a \in \mathsf{A}$ based on an
unobserved parameter $x$ (or a label $c$) given some observation
$y \in Y$. As a general method to measure the performance of this
action, we define a loss function $l(a,x)$, (or $l(a,c)$), which
quantifies the quality of the action $a$ as it relates to the true
outcome $x$ (or $c$). It is well known that the procedure to
minimize expected loss, attaining the Bayes risk, uses the
posterior distribution as follows
\begin{eqnarray}
a^*(y)=\arg\underset{a\in \mathsf{A}}{\min}\; \int_{x} p(x|y)
l(a,x)dx.\label{eqn:postexploss}
\end{eqnarray}
Equation \eqref{eqn:postexploss} tells us how to devise an optimal
action, but it's generally not easy to solve because it's
difficult to perform computation over the posterior distribution
$p(x|y)$.  In special cases, there exist corresponding optimal
actions for particular loss functions: the MAP estimate for the
0-1 loss, the posterior mean for the squared loss, and the
posterior median for the absolute loss. These are special Bayesian
estimates for certain loss functions where point estimates can
provide us the optimal action, but for an arbitrary loss function
we need to solve the optimization problem based on
\eqref{eqn:postexploss}, which is usually challenging. We address
these challenges using our Bayesian inference method.

For binary decision problems considered in this paper, we also
used the receiver operating characteristic (ROC) curve to
visualize the performance across different balances of type-I and
type-II
errors. 

In the following subsections, we demonstrate how our Bayesian
inference method can be used to help devise an optimal action
given observations using simulation and real data.
\subsection{Simulation}
\label{subsec:simulation}

Firstly, we applied this framework to design an optimal action in
the context of sparse signal representations using simulation
data. Suppose that our observation $\mathbf{y}\in \reals^{m}$ is
given as a noisy measurement of a linear forward model. The
standard form of this problem is expressed by

\begin{eqnarray}
\mathbf{y} = \mathbf{M X} + \mathbf{e} \label{eqn:linearmix}
\end{eqnarray}
where $\mathbf{X} \in \reals^{d}$ is the vector of parameters that
are sparse, $\mathbf{M} \in \reals^{m \times d}$ is a matrix for
the linear forward model, and $\mathbf{e}\in \reals^m$ is noise.
This simple model appears in many guises: sparse signal separation
where $\mathbf{M}$ is a mixing matrix \cite{li2006underdetermined,
bell1995information}, and sparse signal representation using
overcomplete dictionaries where $\mathbf{M}$ is a basis matrix
whose columns represent a possible basis \cite{wipf2004sparse}, to
name a few. In this example, we set $m=3$ and $d=3$. To impose the
sparsity model on $\mathbf{X}$, we assumed that the parameters
$\mathbf{X}$ are endowed with a Laplace prior, $p(\mathbf{x})
\propto \exp(-||\mathbf{x}||_1/b)$, and the noise $\mathbf{e}$ is
assumed to be Gaussian with a zero mean, $\mathbf{e}\sim
N(0,\Sigma_e)$. We randomly generated the forward model
$\mathbf{M}\in\reals^{3\times 3}$ from a Gaussian distribution,
and also set the variance parameters for the prior and the
measurement noise as $b=1/\sqrt{2}$ and $\Sigma_e=0.1I$,
respectively.

Given this setting, we aimed to decide whether each component of
the true parameter $\mathbf{x}$ was greater than a preset
threshold $\tau>0$ or not based on the computed posterior and the
MAP estimate, respectively. To achieve this, we designed a loss
function as $l(\mathbf{a},\mathbf{x}) = \sum_{i=j}^{3}
l(a_{j},x_{j})$ where $l(a_{j},x_{j})=1$ was 1 if we made an
incorrect decision; that is, $a_{j}=1$ and $|x_{j}|<\tau$, or
$a_{j}=0$ and $|x_{j}|>\tau$. Otherwise, if the decision was
correct, the loss function was zero.  Here, $\tau$ was chosen to
contain 95 $\%$ of the prior density's mass.

For each simulation, we randomly generated a new $\mathbf{M}$,
$\mathbf{X}$, and $\mathbf{e}$. To devise a Bayesian optimal
decision, we first approximated the expected loss in
\eqref{eqn:postexploss} as \eqref{eqn:conditionalExpectation},
using the posterior samples $Z_i$ drawn from $p(x|y)$ using the
designed optimal map. Then we found $\mathbf{a}$ to minimize the
computed expected loss for each simulation. For MAP decision, we
simply performed \beqa a^*_{MAP} = \argmin_{a=\{0,1\}}
l(a,\mathbf{x}_{\text{MAP}}). \label{eqn:defn:optimalDecision:MAP}
\eeqa Thus the MAP decision was made as $a^*_{MAP}=1$ if
$|x_{\text{MAP}}|>\tau$ and $a^*_{MAP}=0$ otherwise. Then we
computed losses of both decisions by comparing them to the true
$x$ that was used to generate the observation $y$. We repeated
these simulations 200 times. The Bayes decision rule made no
incorrect decisions for all 200 simulations, but the MAP decision
incurred the losses 0, 1, 2, and 3 for 112, 71, 15, and 2
simulations, respectively.

\begin{figure}[t]
\centerline{\includegraphics[width=14cm]{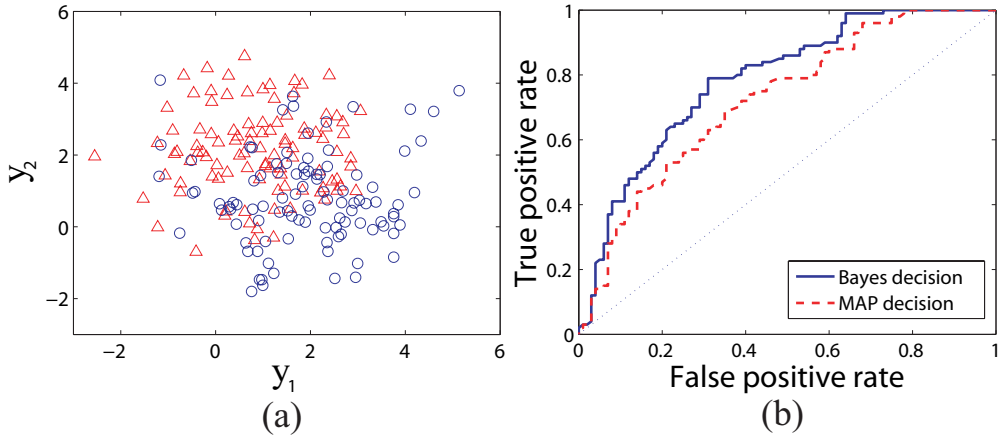}}
\caption{Performance comparison between Bayes and MAP decisions.
(a) Scatter plot of 100 input variables for test. (b) ROC curves
for Bayes and MAP decisions for Bayesian logistic regression
problem with Gaussian prior.} \label{fig:BayesDecisionSim}
\end{figure}

Secondly, we applied this framework to the context of logistic
regression using the simulation data $y$ in binary classes as
described in the third column of Table~\ref{tab:summary}. Logistic
regression has been widely used in many applications, since it is
usually simple and fast to fit a model to data and easy to
interpret the resulting model. In many applications, it is useful
to compute the posterior distribution of the parameters for
logistic regression model. However, there is no conjugate prior,
so posterior calculations can be challenging. Although the
posterior is often approximated as Gaussian, it can deviate much
from the true posterior, due to asymmetry  as shown in
Fig.~\ref{fig:pointvsinterval}~(d). In this example, we instead
computed the posterior distribution for logistic regression using
our proposed Bayesian method.

We demonstrated that the computed posterior helped make a better
decision than an MAP point estimate using samples for a binary
decision problem. We first computed the conditional probability of
a label given observation, $p(c|y)$, and then obtained the ROC
curves by comparing the computed $p(c|y)$ to varying decision
thresholds. The conditional probability is given as
$p(c|y)=\int_{x}p(c|x,y)p(x|y)dx$ where $p(c|x,y)$ is in the form
of sigmoid function. For Bayesian decision, $p(c|y)$ is
approximated by $1/n\sum_{i=1}^{n}p(c|Z_i,y)$ using $Z_i$ drawn
from $p(x|y)$. For the MAP decision, $p(c|y)$ is approximated by
$p(c|x_{\text{MAP}},y)$. We used 10 samples (5 in each class) to
compute the posterior and its MAP estimate. Then we tested Bayes
and MAP decisions using new 100 samples (50 samples in each class)
generated from the same distribution, which were plotted in
Fig.~\ref{fig:BayesDecisionSim}~(a). The Gaussian prior with zero
mean and unit variance was used. The ROC curve in
Fig.~\ref{fig:BayesDecisionSim}~(b) illustrates that the Bayes
decision showed a better decision performance than the MAP
decision.

\subsection{Real data}
\label{subsec:real}

Here, we applied the proposed Bayesian inference framework to real
data sets such as EEG recordings for sleep study
\cite{carskadon2000monitoring} and physiological measurements from
ICU patients.

Firstly, we demonstrate how we estimated the Fourier magnitude
representation of EEG recordings for sleep scoring from a Bayesian
perspective, and then how we used this representation to improve a
sleep monitoring system. Existing methods for sleep scoring
analyzed the relationship of the activity of frequency bands with
sleep stages \cite{kemp2000analysis}. So they relied on the power
spectrum estimate of EEG recording, but largely ignored how
reliable this estimate was. From a statistical viewpoint, the
Fourier representation of a signal can be interpreted as the
maximum likelihood (ML) estimate under a Gaussian noise with zero
mean, and tends to produce many small weights across all
frequencies. However, a sleep EEG signal has special
time-frequency structure, i.e., its power spectrum tends to
concentrate on a certain band or sub-bands depending on sleep
stages, so the standard Fourier transform based approach for sleep
monitoring may not always be an optimal representation for sleep
staging. To address this issue, we put a sparse prior on the
Fourier magnitude spectrum.  By applying an appropriate prior on
the average magnitude spectrum, the spectral analysis in the
Bayesian perspective provides us not only a better representation
of the power spectrum, but also an additional information on our
uncertainty on these estimates.  This allows for an opportunity to
design a better automatic sleep scoring system.

For the Bayesian spectrum analysis we used
PhysioNet\cite{PhysioNet} sleep EEG recording which provides the
recordings together with their hypnograms (a graph that represents
the sleep stages as a function of time) that sleep experts
annotated, and we used these hypnograms as the ground truth of the
decision test. We divided the full-band EEG signal into 8 sub-band
ones to characterize the sleep EEG signal in terms of the power in
each sub-band. A set to include all the frequency components in
the $i$th sub-band was denoted as $\mathbb{F}_i$ for
$i=1,2,...,8$. The sets, $\mathbb{F}_1$, $\mathbb{F}_2$, and
$\mathbb{F}_3$ covered delta (0.5-4 Hz), theta (4-8 Hz), and alpha
(8-12 Hz) bands, $\mathbb{F}_4$ and $\mathbb{F}_5$ covered beta
(12-35 Hz) together, and $\mathbb{F}_6$, $\mathbb{F}_7$, and
$\mathbb{F}_8$ covered the high frequency band (35-50 Hz),
respectively. The sampling frequency was 100 Hz.

The problem settings were described in Table~\ref{tab:summary}. We
then computed the posterior distribution of the \itt{averaged
magnitudes} in each sub-band. Suppose that $y_k$ represented the
$k$th Fourier component of noisy EEG recording with an additive
Gaussian noise with zero mean and $\sigma^2$ variance, and $x_k$
represented the Fourier magnitude of the original EEG signal
before the noise was added. That is, $y_k$ was a complex number
and $x_k$ was a non-negative real number. At the $k$th frequency
bin, the likelihood function for the magnitude spectrum, $x_k$,
was given by the Gaussian distribution with $|y_k|$ mean and
$\sigma^2$ variance, unless the noise was too high
\cite{mcaulay1980speech}. As a selection of prior density, we used
an exponential distribution to impose both sparsity and
non-negativity on average spectrum magnitudes. The input was all
the Fourier components included in all the 8 sub-bands denoted as
$\mathbf{y}$, and the output was the vector of averaged magnitudes
in the 8 sub-bands denoted as
$\mathbf{x}=[x_1,\cdots,x_{8}]^{\text{T}}$. Assuming independence
across the frequency components, the posterior distribution of
$\mathbf{x}$ given $\mathbf{y}$ is expressed as

\begin{eqnarray}
p(\mathbf{x}|\mathbf{y}) \propto
\prod_{i=1}^{8}\prod_{k\in\mathbb{F}_i}\exp\left\{
-\frac{(\abs{y_k}-x_i)^2}{2\sigma^2} \right\}\exp\left\{ -\gamma
x_i\right\}\label{eqn:posteriorsleep2}
\end{eqnarray}
for $x_i\geq 0$. A closed-form representation of the posterior in
(\ref{eqn:posteriorsleep2}) does not exist; we applied our
approach to estimate relevant posterior quantities.
Fig.~\ref{fig:sleepanalysis}~(a) illustrates examples of the
conditional expectation of
$\mathbb{E}[\mathbf{X}|\mathbf{Y}=\mathbf{y}]$, with its 95 $\%$
credible interval for 4 different types of sleep stages.

\begin{figure}[t]
\centerline{\includegraphics[width=14cm]{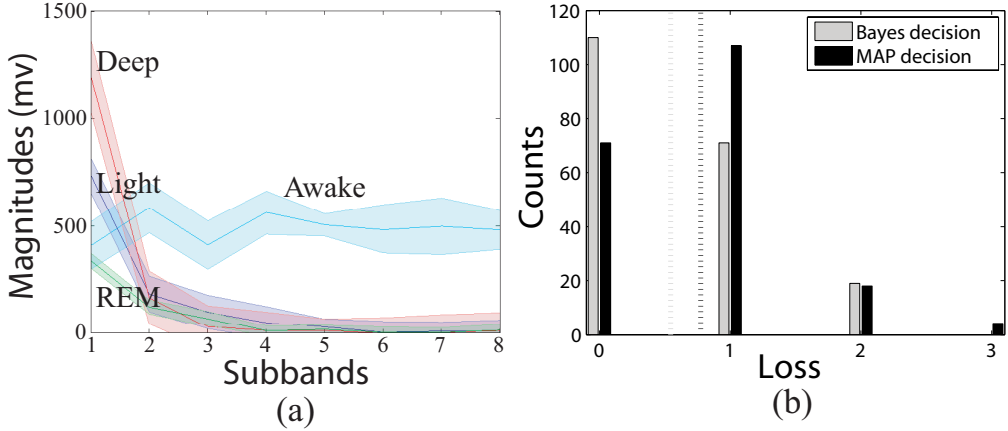}}
\caption{Sleep EEG data analysis. (a) Posterior of the magnitudes
in sub-bands. Each plot represents the posterior mean of the
magnitudes in the sub-bands with their credible intervals.
Different colors represent different sleep stages (cyan: awake,
blue: light, red: deep, and green: REM). (b) Histograms of losses
of Bayes and MAP decisions for sleep stage monitoring. Vertical
dotted lines represent the averaged losses.}
\label{fig:sleepanalysis}
\end{figure}

We next used these posteriors for making decision rules for
automatic sleep scoring. Suppose that $c$ represents one of 4
sleep stages that we need to determine. Our goal is to design an
optimal decision rule to minimize the expected loss in terms of
the posterior distribution of $c$ given $\mathbf{y}$, which is
given by
$p(c|\mathbf{y})=\int_{\mathbf{x}}p(c|\mathbf{x})p(\mathbf{x}|\mathbf{y})d\mathbf{x}$.
Based on the characteristics of each sleep stage described in
\cite{ronzhina2012sleep}, we built a simple model for
$p(c|\mathbf{x})$ as in the Table~\ref{tab:pcx}. We denoted awake,
light, deep, and rapid-eye-movement (REM) sleep stages as W, L, D,
and R for brevity, respectively. The loss function was 3 between
$W$ and $D$, 2 between $W$ and $L$; $W$ and $R$; 1 between others.

\begin{table}
\caption{A rule to design $p(c|\mathbf{x})$. For brevity, awake,
light, deep, and REM sleep stages were denoted by W, L, D, and R,
respectively.} \centerline{
\begin{tabular}{|l|c|c|c|c|}
  \hline
  Conditions of $\mathbf{x}$ & $p(\text{W}|\mathbf{x})$ & $p(\text{L}|\mathbf{x})$ & $p(\text{D}|\mathbf{x})$ & $p(\text{R}|\mathbf{x})$ \\
  \hline
  $>35$ Hz takes $>5\%$ of total.  & 0.8 & 0.1 & 0.05 & 0.05 \\
  \hline
  Total is $<1000$. & 0.1 & 0.2 & 0.05 & 0.65 \\
  \hline
  Delta takes $<60\%$ of total. & 0.1 & 0.6 & 0.2 & 0.1 \\
  \hline
  Delta takes $\geq 60\%$ of total. & 0.05 & 0.3 & 0.6 & 0.05 \\
  \hline
\end{tabular}}\label{tab:pcx}
\end{table}

To evaluate our method, we computed the posterior distributions
for 200 non-overlapping sliding windows. The sleep stages that the
sleep experts manually annotated were provided for every 30
seconds of the EEG recordings \cite{van1990alternative}. We used
the first 5 second of data in each window to make the problem more
challenging. Then we designed an optimal action $a$ given
observation $\mathbf{y}$ for each window to minimize the expected
loss in terms of the posterior distribution.
Fig.~\ref{fig:sleepanalysis}~(b) illustrates the histograms of the
losses for Bayes and MAP decisions for 200 temporal windows. As
illustrated, Bayes decision rule incurred smaller losses than MAP
decision rule.

Next, we applied our framework to develop a risk-prediction system
to predict the survival rates or measure the severity of disease
of ICU patients based on physiological measurements. The
development of this system is helpful for clinical decision
making, standardizing research, comparing the efficacy of
medication or the quality of patient care across ICUs. We used
real physiological measurements of ICU patients together with
their survival outcomes in PhysioNet \cite{PhysioNet}. Since the
outcome $c$ was in binary-class, we designed the optimal action
$a$ to take for prediction based on the Bayesian logistic
regression that we discussed in the previous subsection. The
problem settings were described in the third column of
Table~\ref{tab:summary}.

Using the ICU data, we computed ROC curves for the Bayes and MAP
decisions. Firstly, 10 physiological measurements such as blood
urea nitrogen, Galsgow coma score, heart rate, urine output, etc.,
for 20 subjects were provided with their survival outcomes; For
more details of the physiological measurements, refer to PhysioNet
\cite{PhysioNet}. The survival/non-survival outcome of patients
$c$ was assigned values 1 or 0. Fig.~\ref{fig:ICUanalysis}~(a)
illustrates the scatter plot of two input variables among 10
variables for 100 subjects as an example: blood urea nitrogen and
Galsgow coma score. Fig.~\ref{fig:ICUanalysis}~(a) shows
significant class overlap for these two features (other feature
pairs show similar overlap), suggesting a challenging
classification task.

Given measurements $\mathbf{y}$ and labels $c$ for the 20
subjects, we computed the posterior $p(\mathbf{x}|\mathbf{y})$ of
the unknown parameter $\mathbf{x}$ of the logistic regression
model and estimated $\hat{\mathbf{x}}_{\text{MAP}}$ to maximize
$p(\mathbf{x}|\mathbf{y})$. Then, physiological measurements for
another 100 subjects were provided without their labels. We then
computed the ROC curves in the same manner described in the
Simulation case.

Fig.~\ref{fig:ICUanalysis}~(b) illustrates the ROC curves for the
Bayes and MAP decision rules obtained using 100 subjects. As
shown, the Bayes decision provided us a significantly improved
prediction performance over the MAP decision.

\begin{figure}[t]
\centerline{\includegraphics[width=14cm]{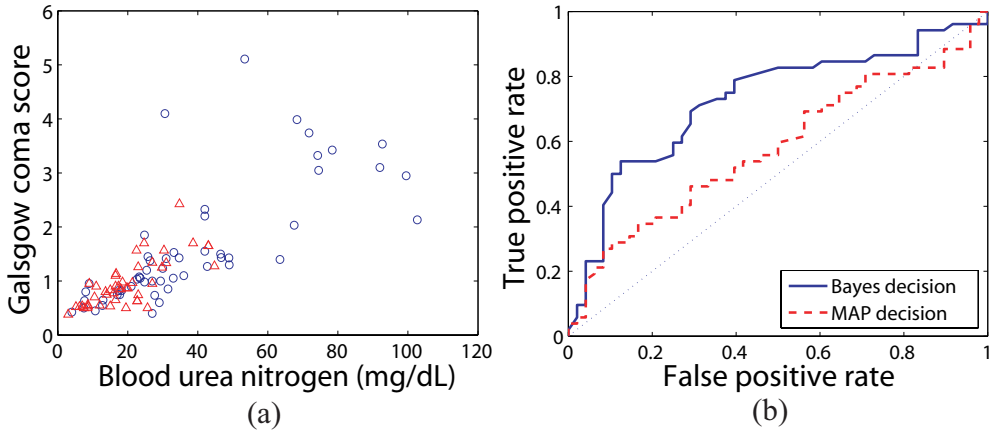}}
\caption{ICU measurements analysis. (a) Scatter plot of two input
variables among 10 variables, blood urea nitrogen and Galsgow coma
score, for 100 subjects. The blue circle and red plus sign
represent survival and non-survival outcomes, respectively. (b)
ROC curves for Bayes and MAP decision for predicting survival in
ICU.} \label{fig:ICUanalysis}
\end{figure}


\section{Discussion}
\label{sec:discussion} We have proposed an efficient Bayesian
inference method based on finding an optimal map which transforms
samples from the prior distribution to samples from the posterior
distribution. Although El Moselhy et al. \cite{el2012bayesian}
proposed the original optimal maps perspective, their formulation
-- in terms of minimizing a variance -- is in general non-convex
and thus computationally intractable.  In this setting, we
considered an alternative approach, based upon KL divergence
minimization, that returns the same optimal solutions. We have
also shown consistency results when using finite-dimensional
approximations that can be implemented computationally. We have
shown that for the class of log-concave priors and likelihoods,
this results in a finite-dimensional convex optimization problem.
We emphasize that the class of log-concave distributions is quite
large and widely used in various applications
\cite{bagnoli2005log}, and that this is the same convexity
condition required for Bayesian point (MAP) estimation. As such,
we have shown that from the perspective of convexity, we can ``get
something for nothing'' by going from point estimation to fully
Bayesian estimation.  Through the optimal map, we demonstrated the
ability to perform computations, with multi-dimensional
parameters, involving the full posterior, including: constructing
Bayesian credible regions, attaining the Bayes risk, drawing
\textit{i.i.d.} samples from the posterior, and generating ROC
curves.

Other applications outside of Bayesian inference might be able to
benefit from this approach. In Section~\ref{sec:gen-push-thrm}, we
demonstrated a more general result about transforming samples from
$P$ to $Q$ whenever $Q$ is log-concave and $P$ can be easily
sampled from. Outside the scope of Bayesian inference (where $P$
is the prior and $Q$ is the posterior), this ability may have
applications including but are not limited to data compression
\cite{cover2006elements}, and message-point feedback information
theory \cite{ma2011generalizing}.

Although we have established convexity of these schemes, further
work can be done in developing parallelized optimization
algorithms that modern large-scale machine architectures routinely
use for Bayesian point estimation. Characterizing the fundamental
limits of sample complexity of this approach of Bayesian inference
help guide how these architectures may possibly be soundly
implemented.  Optimizing architectures for hardware optimization,
and understanding performance-energy-complexity tradeoffs, will
further allow for wider exploration of these methods within the
context of emerging applications, such as wearables
\cite{kim2011epidermal,kang2015scalable} and the
internet-of-things \cite{atzori2010internet}.

\bibliographystyle{IEEEtran}
\bibliography{BayesianInference}

\begin{thebibliography}{10}
\providecommand{\url}[1]{#1}
\csname url@samestyle\endcsname
\providecommand{\newblock}{\relax}
\providecommand{\bibinfo}[2]{#2}
\providecommand{\BIBentrySTDinterwordspacing}{\spaceskip=0pt\relax}
\providecommand{\BIBentryALTinterwordstretchfactor}{4}
\providecommand{\BIBentryALTinterwordspacing}{\spaceskip=\fontdimen2\font plus
\BIBentryALTinterwordstretchfactor\fontdimen3\font minus
  \fontdimen4\font\relax}
\providecommand{\BIBforeignlanguage}[2]{{%
\expandafter\ifx\csname l@#1\endcsname\relax
\typeout{** WARNING: IEEEtran.bst: No hyphenation pattern has been}%
\typeout{** loaded for the language `#1'. Using the pattern for}%
\typeout{** the default language instead.}%
\else
\language=\csname l@#1\endcsname
\fi
#2}}
\providecommand{\BIBdecl}{\relax}
\BIBdecl

\bibitem{bagnoli2005log}
M.~Bagnoli and T.~Bergstrom, ``Log-concave probability and its applications,''
  \emph{Economic theory}, vol.~26, no.~2, pp. 445--469, 2005.

\bibitem{walker2003defining}
W.~E. Walker, P.~Harremo{\"e}s, J.~Rotmans, J.~P. van~der Sluijs, M.~B. van
  Asselt, P.~Janssen, and M.~P. Krayer~von Krauss, ``Defining uncertainty: a
  conceptual basis for uncertainty management in model-based decision
  support,'' \emph{Integrated assessment}, vol.~4, no.~1, pp. 5--17, 2003.

\bibitem{degroot1970optimal}
M.~H. DeGroot, \emph{Optimal statistical decisions}.\hskip 1em plus 0.5em minus
  0.4em\relax McGraw-Hill, 1970.

\bibitem{raginsky2009mutual}
M.~Raginsky and T.~P. Coleman, ``Mutual information and posterior estimates in
  channels of exponential family type,'' in \emph{IEEE Inf Theory Workshop},
  2009, pp. 399--403.

\bibitem{degroot1962uncertainty}
M.~H. DeGroot, ``Uncertainty, information, and sequential experiments,''
  \emph{The Annals of Math Stat}, pp. 404--419, 1962.

\bibitem{cover2006elements}
T.~Cover and J.~Thomas, \emph{{Elements of information theory}}.\hskip 1em plus
  0.5em minus 0.4em\relax Wiley-Interscience, 2006.

\bibitem{koller2009probabilistic}
D.~Koller and N.~Friedman, \emph{Probabilistic graphical models: principles and
  techniques}.\hskip 1em plus 0.5em minus 0.4em\relax MIT press, 2009.

\bibitem{quinn2015directedInformationGraphs}
C.~J. Quinn, N.~Kiyavash, and T.~P. Coleman, ``Directed information graphs,''
  \emph{IEEE Tran Inf Theory}, 2015, to appear.

\bibitem{kass1988asymptotics}
R.~Kass, L.~Tierney, and J.~Kadane, ``Asymptotics in bayesian computation,''
  \emph{Bayesian statistics}, vol.~3, pp. 261--278, 1988.

\bibitem{geisser1990validity}
Geisser \emph{et~al.}, ``The validity of posterior expansions based on
  laplace's method,'' \emph{Bayesian and likelihood meth in stat and economet:
  Essays in honor of George A. Barnard}, vol.~7, p. 473, 1990.

\bibitem{good1980some}
I.~J. Good, ``Some history of the hierarchical {B}ayesian methodology,''
  \emph{Trab estad{\'\i}stica y invest operat}, vol.~31, no.~1, 1980.

\bibitem{goldstein2011multilevel}
H.~Goldstein, \emph{Multilevel statistical models}.\hskip 1em plus 0.5em minus
  0.4em\relax John Wiley \& Sons, 2011.

\bibitem{tzikas2008variational}
D.~G. Tzikas, C.~Likas, and N.~P. Galatsanos, ``The variational approximation
  for {B}ayesian inference,'' \emph{IEEE Signal Process Mag}, vol.~25, no.~6,
  pp. 131--146, 2008.

\bibitem{djuric2003particle}
Djuric \emph{et~al.}, ``Particle filtering,'' \emph{IEEE Signal Process Mag},
  vol.~20, no.~5, pp. 19--38, 2003.

\bibitem{arulampalam2002tutorial}
M.~S. Arulampalam, S.~Maskell, N.~Gordon, and T.~Clapp, ``A tutorial on
  particle filters for online nonlinear/non-{G}aussian {B}ayesian tracking,''
  \emph{IEEE Tran on Signal Process}, vol.~50, no.~2, pp. 174--188, 2002.

\bibitem{walker1999bayesian}
S.~G. Walker, P.~Damien, P.~W. Laud, and A.~F. Smith, ``Bayesian nonparametric
  inference for random distributions and related functions,'' \emph{J Royal
  Stat Society: Series B (Stat Meth)}, vol.~61, no.~3, 1999.

\bibitem{neal2000markov}
R.~M. Neal, ``Markov chain sampling methods for dirichlet process mixture
  models,'' \emph{J comput and graph stat}, vol.~9, no.~2, 2000.

\bibitem{muller2004nonparametric}
P.~M{\"u}ller and F.~A. Quintana, ``Nonparametric {B}ayesian data analysis,''
  \emph{Stat sci}, vol.~19, no.~1, pp. 95--110, 2004.

\bibitem{teh2006hierarchical}
Y.~W. Teh, M.~I. Jordan, M.~J. Beal, and D.~M. Blei, ``Hierarchical {D}irichlet
  processes,'' \emph{J the american stat assoc}, vol. 101, no. 476, 2006.

\bibitem{robert2004monte}
C.~P. Robert and G.~Casella, \emph{Monte Carlo statistical methods}.\hskip 1em
  plus 0.5em minus 0.4em\relax Springer, 2004.

\bibitem{andrieu2003introduction}
C.~Andrieu, N.~De~Freitas, A.~Doucet, and M.~I. Jordan, ``An introduction to
  {MCMC} for machine learning,'' \emph{Mach learn}, vol.~50, no. 1-2, 2003.

\bibitem{hastings1970monte}
W.~K. Hastings, ``Monte carlo sampling methods using {M}arkov chains and their
  applications,'' \emph{Biometrika}, vol.~57, no.~1, pp. 97--109, 1970.

\bibitem{geman1984stochastic}
S.~Geman and D.~Geman, ``Stochastic relaxation, gibbs distributions, and the
  bayesian restoration of images,'' \emph{IEEE Trans Pattern Anal Mach Intell},
  no.~6, pp. 721--741, 1984.

\bibitem{Liu2008}
J.~S. Liu, \emph{Monte Carlo Strategies in Scientific Computing}.\hskip 1em
  plus 0.5em minus 0.4em\relax Springer, 2008.

\bibitem{jordan1998introduction}
M.~I. Jordan, Z.~Ghahramani, T.~S. Jaakkola, and L.~K. Saul, ``An introduction
  to variational methods for graphical models,'' \emph{Machine learning},
  vol.~37, no.~2, 1999.

\bibitem{jaakkola2001tutorial}
T.~S. Jaakkola, ``Tutorial on variational approximation methods,'' in \emph{In
  Advanced Mean Field Methods: Theory and Practice}.\hskip 1em plus 0.5em minus
  0.4em\relax MIT Press, 2000.

\bibitem{bishop2006pattern}
C.~M. Bishop, \emph{Pattern recognition and machine learning}.\hskip 1em plus
  0.5em minus 0.4em\relax springer, 2007.

\bibitem{minka2001expectation}
T.~P. Minka, ``Expectation propagation for approximate {B}ayesian inference,''
  in \emph{Uncertainty in artif intell}, 2001, pp. 362--369.

\bibitem{minka2001family}
------, ``A family of algorithms for approximate {B}ayesian inference,'' Ph.D.
  dissertation, Massachusetts Institute of Technology, 2001.

\bibitem{seeger2008bayesian}
M.~W. Seeger, ``Bayesian inference and optimal design for the sparse linear
  model,'' \emph{J Mach Learn Research}, vol.~9, pp. 759--813, 2008.

\bibitem{jordan1999introduction}
M.~I. Jordan, Z.~Ghahramani, T.~S. Jaakkola, and L.~K. Saul, ``An introduction
  to variational methods for graphical models,'' \emph{Mach learn}, vol.~37,
  no.~2, pp. 183--233, 1999.

\bibitem{el2012bayesian}
T.~El~Moselhy and Y.~Marzouk, ``Bayesian inference with optimal maps,'' \emph{J
  Comput Physics}, vol. 231, no.~23, pp. 7815--7850, 2012.

\bibitem{monge1781}
G.~Monge, \emph{M{\'e}moire sur la th{\'e}orie des d{\'e}blais et des
  remblais}.\hskip 1em plus 0.5em minus 0.4em\relax De l'Imprimerie Royale,
  1781.

\bibitem{kantorovich1942mass}
L.~Kantorovich, ``On mass transportation,'' in \emph{Dokl. Akad. Nauk. SSSR},
  vol.~37, 1942, pp. 227--229.

\bibitem{villani2003topics}
C.~Villani, \emph{Topics in optimal transportation}.\hskip 1em plus 0.5em minus
  0.4em\relax AMS, 2003.

\bibitem{kim2013efficient}
M.~R. M.~D. Kim, S. and T.~Coleman, ``Efficient {B}ayesian inference methods
  via convex optimization and optimal transport,'' in \emph{IEEE Intern Symp
  Inf Theory}, July 2013, pp. 2259--2263.

\bibitem{xiu2003wiener}
D.~Xiu and G.~Karniadakis, ``{The Wiener-Askey polynomial chaos for stochastic
  differential equations},'' \emph{{SIAM journal on sci comput}}, vol.~24,
  no.~2, pp. 619--644, 2003.

\bibitem{ernstConvergencePolynomialChaos2012}
O.~G. Ernst, A.~Mugler, H.~Starkloff, and E.~Ullman, ``On the convergence of
  generalized polynomial chaos expansions,'' \emph{ESAIM: Math Model and Num
  Analy}, vol.~46, pp. 317--339, 2012.

\bibitem{ParikhBoydProximalAlgs2013}
N.~Parikh and S.~Boyd, ``Proximal algorithms,'' \emph{Foundations and Trends in
  Opt}, vol.~1, no.~3, pp. 131--231, 2013.

\bibitem{grant2008cvx}
\BIBentryALTinterwordspacing
M.~Grant, S.~Boyd, and Y.~Ye, ``Cvx: Matlab software for disciplined convex
  programming.'' [Online]. Available: \url{http://cvxr.com/cvx/}
\BIBentrySTDinterwordspacing

\bibitem{Berk72ConsistencyNormality}
R.~H. Berk, ``{Consistency and Asymptotic Normality of MLE's for Exponential
  Models},'' \emph{{The Annals of Math Stat}}, vol.~43, no.~1, 1972.

\bibitem{boyd2004convex}
S.~Boyd and L.~Vandenberghe, \emph{Convex optimization}.\hskip 1em plus 0.5em
  minus 0.4em\relax Cambridge University Press, 2004.

\bibitem{li2006underdetermined}
Y.~Li, S.-I. Amari, A.~Cichocki, D.~W. Ho, and S.~Xie, ``Underdetermined blind
  source separation based on sparse representation,'' \emph{IEEE Tran Signal
  Process}, vol.~54, no.~2, pp. 423--437, 2006.

\bibitem{bell1995information}
A.~J. Bell and T.~J. Sejnowski, ``An information-maximization approach to blind
  separation and blind deconvolution,'' \emph{Neural comput}, vol.~7, no.~6,
  pp. 1129--1159, 1995.

\bibitem{wipf2004sparse}
D.~P. Wipf and B.~D. Rao, ``Sparse {B}ayesian learning for basis selection,''
  \emph{IEEE Tran Signal Process}, vol.~52, no.~8, pp. 2153--2164, 2004.

\bibitem{carskadon2000monitoring}
M.~A. Carskadon and A.~Rechtschaffen, ``Monitoring and staging human sleep,''
  \emph{Principles and practice of sleep medicine}, vol.~3, 2000.

\bibitem{kemp2000analysis}
B.~Kemp, A.~H. Zwinderman, B.~Tuk, H.~A. Kamphuisen, and J.~J. Oberye,
  ``Analysis of a sleep-dependent neuronal feedback loop: the slow-wave
  microcontinuity of the {EEG},'' \emph{IEEE Tran Biomed Eng}, vol.~47, no.~9,
  pp. 1185--1194, 2000.

\bibitem{PhysioNet}
A.~L. Goldberger \emph{et~al.}, ``{PhysioBank, PhysioToolkit, and PhysioNet}:
  Components of a new research resource for complex physiologic signals,''
  \emph{Circulation}, vol. 101, no.~23, pp. e215--e220, 2000.

\bibitem{mcaulay1980speech}
R.~McAulay and M.~Malpass, ``Speech enhancement using a soft-decision noise
  suppression filter,'' \emph{IEEE Tran Acoustics, Speech and Signal Process},
  vol.~28, no.~2, pp. 137--145, 1980.

\bibitem{ronzhina2012sleep}
M.~Ronzhina, O.~Janou{\v{s}}ek, J.~Kol{\'a}{\v{r}}ov{\'a}, M.~Nov{\'a}kov{\'a},
  P.~Honz{\'\i}k, and I.~Provazn{\'\i}k, ``Sleep scoring using artificial
  neural networks,'' \emph{Sleep Medicine Reviews}, vol.~16, no.~3, pp.
  251--263, 2012.

\bibitem{van1990alternative}
B.~Van~Sweden, B.~Kemp, H.~Kamphuisen, and E.~Van~der Velde, ``Alternative
  electrode placement in (automatic) sleep scoring ({F}pz-{C}z/{P}z-{O}z versus
  {C}4-{A}1),'' \emph{Sleep}, vol.~13, no.~3, pp. 279--283, 1990.

\bibitem{ma2011generalizing}
R.~Ma and T.~Coleman, ``Generalizing the posterior matching scheme to higher
  dimensions via optimal transportation,'' in \emph{Allerton}, 2011.

\bibitem{kim2011epidermal}
D.-H. Kim, N.~Lu, R.~Ma, Y.-S. Kim, R.-H. Kim, S.~Wang, J.~Wu, S.~M. Won,
  H.~Tao, A.~Islam \emph{et~al.}, ``Epidermal electronics,'' \emph{Science},
  vol. 333, no. 6044, pp. 838--843, 2011.

\bibitem{kang2015scalable}
D.~Y. Kang, Y.-S. Kim, G.~Ornelas, M.~Sinha, K.~Naidu, and T.~P. Coleman,
  ``Scalable microfabrication procedures for adhesive-integrated flexible and
  stretchable electronic sensors,'' \emph{Sensors}, vol.~15, no.~9, pp.
  23\,459--23\,476, 2015.

\bibitem{atzori2010internet}
L.~Atzori, A.~Iera, and G.~Morabito, ``The internet of things: A survey,''
  \emph{Computer networks}, vol.~54, no.~15, pp. 2787--2805, 2010.

\end{thebibliography}

\end{document}